\documentclass[12pt]{article}

\usepackage[title]{appendix}
\usepackage{amsmath}
\usepackage{amsthm}
\usepackage{amssymb}
\usepackage{bbm}
\usepackage{verbatim}
\usepackage{graphicx}
\usepackage{subfigure}
\usepackage{afterpage}
\usepackage{etoolbox}
\usepackage{float}
\usepackage{algorithmic}
\usepackage[lined,boxed,linesnumbered,algoruled,vlined]{algorithm2e}
\usepackage{multirow}
\usepackage{rotating}
\usepackage{epstopdf}
\usepackage{multirow}
\usepackage{xcolor,colortbl}
\usepackage[font={small}]{caption}

\usepackage{geometry}
\numberwithin{equation}{section}
\numberwithin{figure}{section}
\numberwithin{table}{section}

\makeatletter
\renewcommand{\p@subfigure}{\thefigure}
\makeatother

\newtheorem{definition}{Definition}[section]
\newtheorem{theorem}{Theorem}[section]
\newtheorem{proposition}[theorem]{Proposition}
\newtheorem{lemma}[theorem]{Lemma}

\newcommand{\repeatable}[2]{\makeatletter \global\expandafter\def\csname repText@#1\endcsname {#2} \makeatother #2}
\newcommand{\repeatxt}[1]{\makeatletter \expandafter\csname repText@#1\endcsname \makeatother}

\newtoggle{citeInAbstract}
\toggletrue{citeInAbstract}

\newcommand{\usecrop}[2]
{
	\newlength{\cropwidth}
	\setlength{\cropwidth}{\the\textwidth}
	\addtolength{\cropwidth}{#1}
	\newlength{\cropheight}
	\setlength{\cropheight}{\the\textheight}
	\addtolength{\cropheight}{#2}
	\usepackage[width=\the\cropwidth,height=\the\cropheight,center]{crop}
}

\DeclareMathAlphabet{\mathpzc}{OT1}{pzc}{m}{it}

\newcommand{\abs}[1]{\left | #1 \right |}
\newcommand{\absinline}[1]{| #1 |}
\newcommand{\norm}[1]{\left \| #1 \right \|}
\newcommand{\norminline}[1]{\| #1 \|}

\newcommand{\Rn}[1]{{\mathbbm{R}^{#1}}}
\newcommand{\ve}{\varepsilon}

\makeatletter
\newcommand{\thickhline}{%
    \noalign {\ifnum 0=`}\fi \hrule height 1pt
    \futurelet \reserved@a \@xhline
}
\newcolumntype{"}{@{\hskip\tabcolsep\vrule width 1pt\hskip\tabcolsep}}
\makeatother

\newcommand{\set}{\mathcal}
\newcommand{\vect}{\mathbf}
\newcommand{\ivect}[2]{{\mathbf{#1}}^{({#2})}}
\newcommand{\iscal}[2]{{#1}_{{#2}}}
\newcommand{\map}{\mathbf}
\newcommand{\submat}[3]{{#1}_{({#2},{#3})}}
\newcommand{\rows}[2]{\submat{{#1}}{{#2}}{:}}
\newcommand{\cols}[2]{\submat{{#1}}{:}{{#2}}}
\newcommand{\ventry}[2]{{\vect{#1}}_{(#2)}}
\newcommand{\mentry}[3]{{#1}_{(#2,#3)}}
\newcommand{\bmu}{\boldsymbol\mu}
\newcommand{\mapi}{\boldsymbol\pi}
\newcommand{\proj}[1]{{\map{W}}_{#1}}
\newcommand{\piproj}[1]{\proj{\mapi([#1])}}
\newcommand{\perproj}[1]{{\proj{#1}^\perp}}
\newcommand{\perpiproj}[1]{{\piproj{#1}^\perp}}
\newcommand{\Wpik}[1]{\set{A}_{\mapi([#1])}}

\newcommand{\timel}[2]{{#1}^{(#2)}}

\title{Incomplete Pivoted QR-based Dimensionality Reduction}
\author{Amit Bermanis, Aviv Rotbart, Moshe Salhov and Amir Averbuch\\
School of Computer Science\\
Tel Aviv University, Tel Aviv 69978, Israel}
 \date{\today}
\begin{document}
\maketitle
\begin{abstract}
High-dimensional big data appears in many research fields such as image recognition, biology and collaborative filtering. Often, the exploration of such data by classic algorithms is encountered with difficulties due to `curse of dimensionality' phenomenon. Therefore, dimensionality reduction methods are applied to the data prior to its analysis. Many of these methods are based on principal components analysis, which is statistically driven, namely they map the data into a low-dimension subspace that preserves significant statistical properties of the high-dimensional data. As a consequence, such methods do not directly address the geometry of the data, reflected by the mutual distances between multidimensional data point. Thus, operations such as classification, anomaly detection or other machine learning tasks may be affected.

This work provides a dictionary-based framework for geometrically driven data analysis that includes dimensionality reduction, out-of-sample extension and anomaly detection. It embeds high-dimensional data in a low-dimensional subspace. This embedding preserves the original high-dimensional geometry of the data up to a user-defined distortion rate. In addition, it identifies a subset of landmark data points that constitute a dictionary for the analyzed dataset. The dictionary enables to have a natural extension of the low-dimensional embedding to out-of-sample data points, which gives rise to a distortion-based criterion for anomaly detection. The suggested method is demonstrated on synthetic and real-world datasets and achieves good results for classification, anomaly detection and out-of-sample tasks. 
\end{abstract}

\par\noindent
{\bf Key words}: linear dimensionality reduction, incomplete pivoted QR, distortion, out-of-sample extension, user-defined distortion, diffusion maps

\section{Introduction}\label{sec:intro}
Nowadays, continuous sampling of measurements from sensor systems of real-world processes has generate ever-growing datasets. Analysis of data silos is a fundamental task in many scientific and
industrial fields, whose goal is to infer significant information from a collection of observations/measurements that illuminate the underlying phenomenon that generates the observed data. The illuminated observations  will assist in tasks such as classification, clustering, forecasting and anomaly detection to name some.

High-dimensional big data analysis is of special interest since multidimensional data points usually reside in a lower dimensional subspace of the ambient space. For example, data clustering requires regions of high density that constitute clusters. In high dimensions, a huge number of data points is required to create high density regions and this number grows exponentially with dimension. 

High-dimensional data appear in two main forms: parametric and non-parametric. In the former case, every observed data point consists of many parameters where each corresponds to a single dimension. Typically, parametric data is concerned with the geometry of such data. Non-parametric data is typically originated in artificial geometry ascription to the analyzed data that  usually encapsulated in a \emph{kernel} matrix. Analysis of this geometry, through the analysis of the associated kernel, can uncover latent data features  alas such geometries may be high-dimensional. Either way, for efficient analysis of high-dimensional data, a dimensionality reduction, which  preserves the original high-dimensional geometry, is needed.

Principal Component Analysis (PCA)~\cite{hotelling:PCA} is a statistically driven linear dimensionality reduction method, originally designates for the analysis of parametric data. It acts on the Gram matrix of the data and embeds the data in a low-dimensional space, whose coordinates are the directions of the high variances of the data that are also known as principal components. PCA can be implemented by the application of the singular value decomposition (SVD)~\cite{golub} to a data matrix. However, this implementation is computationally expensive. Section~\ref{sec: svd} details the essentials of SVD.

Induced by PCA for linear analysis of parametric data, the kernel PCA  approach for non-linear analysis of non-parametric data was introduced in~\cite{kernelpca}. Essentially, if the kernel matrix is semi-positive definite, then it can be treated as a Gram matrix of the data in the high-dimensional space, which is also referred to as \emph{feature space}. Then, PCA is applied to the kernel matrix in order to embed the data in a low-dimensional space. Some examples for kernel methods, among many others, are local linear embedding (LLE)~\cite{roweis2000nonlinear}, Isomap~\cite{isomap} and Diffusion Maps (DM)~\cite{coifman:DM2006}. LLE embeds the data in  a low-dimensional space whose geometry represents local linear geometries of the high-dimensional data. The low-dimensional geometry, which is produced by Isomap, preserves the geodesic distances in the original data. DM provides a low-dimensional representation of the diffusion geometry of the data as defined in~\cite{berard,coifman:DM2006}.

In this work, an incomplete pivoted QR-based deterministic method for dimensionality reduction is presented. The method is designed to preserve a high-dimensional Euclidean geometry of parametric data up to a user-specified distortion rate, according to the following definition:
\begin{definition}[$\mu$-distortion] \label{def: distortion}
Let $(\set H, \map m_\set{H})$ and $(\set L, \map m_\set{L})$ be metric spaces, $\set A\subset \set{H}$ and $\mu\geq 0$. A map $\map F:\set A\to\set L$ is called a $\mu$-distortion of $\set A$ if $\sup_{x,y\in\set A} \abs{\map m_\set{H}(x,y)-\map m_\set{L}(\map F(x),\map F(y))}\leq\mu$. The space $\set L$ is referred to as the $\mu$-embedding space of $\set A$.
\end{definition}
The proposed method is dictionary-based, where the dictionary is chosen from the analyzed dataset $\set A$. The method identifies a Euclidean embedding space, which is spanned by the dictionary members, on which the orthogonal projection of the data provides a user-defined distortion of the original high-dimensional dataset. In that sense, our method is geometrically driven as opposed to PCA. Clearly, there is an interplay between the distortion rate and the dimension of the resulted embedding subspace: the smaller $\mu$ is the higher is  the embedding's dimension and vice-versa. Our method preserves global patterns of the data and trades local geometry for low-dimensional representation since dense regions in the original data (such as clusters) are more sensitive to distortions than sparse regions such as gaps between clusters.

Additionally to dimensionality reduction, we present two strongly related schemes for out-of-sample extension and anomaly detection, which are naturally stem from the proposed method for dimensionality reduction. Thus, the dimensionality reduction phase, followed by an out-of-sample extension and by an anomaly detection, constitutes a complete framework for semi-supervised learning, where the original (in-sample) dataset $\set A$ functions as a training set. In this context, the learning phase is reflected in the extraction of a $\mu$-embedding subspace of $\set A$ as defined in Definition~\ref{def: distortion}. Out-of-sample data points, whose projection on the representative subspace is of low distortion, are classified as normal, while the rest are classified as abnormal. Therefore, the original dataset $\set A$ is considered as normal by definition.

To conclude, the contribution of this work is threefold: first, the suggested method identifies landmark data points (dictionary) that represent the data, as opposed to PCA that lacks this property, and therefore in some sense it is less informative. In matrix decomposition terminology, this is known as the Columns Subset Selection (CSS) problem. Secondly, the presented method requires very low storage budget relatively to PCA consumption. In the worst case, its computational complexity is identical to that of PCA computation. Lastly, the proposed out-of-sample extension and anomaly detection constitute natural extensions to the dimensionality reduction phase.

The rest of this paper is organized as follows: in the rest of this section, we review related works (Section~\ref{sec:related_works}), describe the used notation and our general approach (Sections~\ref{sec:notation} and~\ref{sec:genappch}, respectively), and discuss the essentials of PCA-based dimensionality reduction (Section~\ref{sec: svd}). Section~\ref{sec:qrdr} establishes the theory and the technical tools on which our method is based. It presents the robustness of our method to noise and its relation to matrix approximation, as well. Although the proposed method is designated for parametric data analysis, section~\ref{sec:qrdm} describes its utilization for Diffusion Maps (DM), which is a non-parametric analysis method. Section~\ref{sec:experimental} presents the experimental results. Finally, Section~\ref{sec:conclusions} concludes the paper and discusses future work.

\subsection{Related works}\label{sec:related_works}
Johnson-Lindenstrauss (JL) Lemma~\cite{johnson1984extensions} constitutes a basis for many random projection based dimensionality reduction methods~\cite{linial1995geometry,indyk1998approximate,schulman2000clustering,achlioptas2003database,ailon2006approximate,clarkson2008tighter,baraniuk2008simple,santosh2005random} to name some. In contrast to the absolute bound in Definition~\ref{def: distortion}, JL Lemma guarantees a relative distortion bound of the form $(1-\varepsilon)\cdot \map m^2_\set{H}(x,y)\leq \map m^2_\set{L}(\map F(x),\map F(y))\leq (1+\varepsilon)\cdot \map m^2_\set{H}(x,y) $ with high probability. From data analysis perspective, the absolute bound may be more useful when the data is comprised of dense clusters separated by sparse regions. In this scenario, the absolute bound guarantees embedding such that intra-cluster distances may be distorted but inter-cluster distances are preserved so that the global high-dimensional geometry is preserved in the $\mu$-embedded space. The relative bound, on the other hand, may produce a low-dimensional space in which clusters are becoming too close to each other.

Another significant branch of dimensionality reduction methods, which is strongly related to the presented work, deals with the CSS problem~\cite{frieze2004fast, Drineas06subspacesampling, Boutsidis, Boutsidis:2009,boutsidis2008unsupervised,mahoney2011randomized} to name a few. Interpolative decomposition (ID) of a matrix was introduced first for operator compression~\cite{cheng2005compression}. It is designed  to approximate spectrally linear integral operators. A randomized version of ID is presented in~\cite{martinsson2011randomized}. The class of randomized matrix decomposition algorithms is out of the scope of this paper. CUR decomposition~\cite{boutsidis2014optimal,mahoney2009cur} of a given matrix $A$ generalizes the ID in the sense  that both subsets of columns and rows of the original matrix are selected to form the matrices $C$ and $R$, respectively, such that $A\approx CUR$ for a low rank matrix $U$. In data analysis terms, this decomposition enables us to get sampling of significant data points (rows), as well as significant features (columns).  Similar to ID and unlike this work, CUR decompositions are designated to spectrally approximate the original matrix $A$. Randomized versions of CUR also exist - see~\cite{wang2013improving,drineas2008relative} and the references therein.

\subsection{Notation}\label{sec:notation}
In the rest of the paper, the following notation are used: for $k\in\mathbbm{N}$, $[k]=\{1,\ldots,k\}$. $I_n$ is the $n\times n$ unit matrix. The $i$-th coordinate of a vector $\vect{v}\in\Rn{n}$ is denoted by $\ventry{v}{i}$. The $(i,j)$-th entry of a matrix $A$ is $\mentry{A}{i}{j}$ and its $i$-th row and $j$-th column are $\rows{A}{i}$ and $\cols{A}{j}$, respectively. If $A$ is of size $m\times n$ and $\set I=\{i_1,\ldots,i_k\}\subset [m]$ and $\set J=\{j_1,\ldots,j_\ell\}\subset[n]$ are two ordered sets, then $\rows{A}{\set I}$ is the $k\times n$ matrix $B$ for which $\rows{B}{r} = \rows{A}{i_r}$, $r\in[k]$, $\cols{A}{\set J}$ is the $m\times \ell$ matrix $C$ for which $\cols{C}{r} = \cols{A}{j_r}$, $r\in[\ell]$ and $\submat{A}{\set I}{\set J}$ is the $k\times\ell$ matrix $D$, whose $(p,q)$-th entry is $\mentry{D}{p}{q} = \mentry{A}{i_p}{j_q},~p\in[k],q\in[\ell]$. The transposed matrix of $A$ is denoted by $A^\ast$ and $A^\dagger$ is the Moore-Penrose pseudo-inverse of $A$. $\mapi:[n]\to[n]$ denotes permutation, and $\Pi$ is the associated $n\times n$ permutation matrix such that $\mentry{\Pi}{i}{j}=1$ if $\mapi(j)=i$, otherwise $\mentry{\Pi}{i}{j}=0$. For a subspace $\set S\subset\Rn{m}$, $\set S^\perp$ is the complementary perpendicular subspace of $\set S$ in $\Rn{m}$, and $\proj{\set S},\perproj{\set S}:\Rn{m}\to\Rn{m}$ are the corresponding orthogonal projections on these subspaces, respectively. Finally, $\norminline{\vect{v}} $ is the standard Euclidean norm of $\vect{v}$, where $\norminline{\vect v}_{1} = \sum_{i=1}^n\absinline{\ventry{v}{i}}$ is its $\ell_1$ norm.

The explored dataset is $\set A = \{\ivect{a}{1},\ldots,\ivect{a}{n}\}\subset\Rn{m}$ and the associated data matrix is the $m\times n$ matrix $A$, whose $j$-th column is $\cols{A}{j}=\ivect{a}{j}$, $j\in[n]$.

\subsection{General approach}\label{sec:genappch}

Our approach to achieve a low-rate distortion embedding consists of two main steps:
\begin{enumerate}
\item\label{cond1} Given a nonnegative distortion parameter $\mu$. An $s$-dimensional $2\mu$-embedded subspace $\set S \subset \Rn{m}$ is identified, for which the orthogonal  projection of any $\vect a\in\set A$ results in an energy loss of at most $\mu$, i.e.
\begin{equation}
\label{eq:cond1}
\norm{\perproj{\set S}(\ivect{a}{i})} \leq\mu,\quad i\in[n].
\end{equation}
Here, the nonnegative distortion parameter $\mu$ is a user defined input. Clearly, $s$ is a non-increasing function of $\mu$.
\item\label{cond2} The subspace $\set S$ is orthogonally aligned with $\Rn{s}$ to achieve an $s$-dimensional representation of $\set A$, i.e. $\map O_{\set S}:\Rn{m}\to\Rn{s}$ is an orthogonal transformation that satisfies
\begin{equation}
\label{eq:cond2}
\norminline{\map O_{\set S}\vect v} =\norminline{\vect v} ,\quad\vect v\in\set S.
\end{equation}
Obviously, such an alignment (which is not unique) does not affect the geometry of the projected set $\set A$ on $\set S$.
\end{enumerate}
Application of the above two-stage scheme to $\set A$ results in a $2\mu$-distortion as Lemma~\ref{prop:fund} shows.
\begin{lemma}   \label{prop:fund}
Let $\set S\subset\Rn{m}$ be an $s$-dimensional subspace of $\Rn{m}$ that satisfies Step~1 and let $\map O_{\set S}:\Rn{m}\to\Rn{s}$ be an orthogonal transformation that satisfies Step~2. Then, the $s$-dimensional map $\map F_s: \Rn{m}\to\Rn{s}$ \begin{equation}
\label{eq:Fdecomp}
\map F_s\triangleq \map O_{\set S}\circ\proj{\set S}
\end{equation}
is a $2\mu$-distortion of $\set A$.
\end{lemma}

\begin{proof}
From Eqs.~\ref{eq:cond2} and~\ref{eq:Fdecomp} we get $\norminline{\map F_s(\vect v)} = \norminline{\map O_{\set S}\circ\proj{\set S}(\vect v)}= \norminline{\proj{\set S}(\vect v)}$. Since $\proj{\set S}$ is orthogonal projection, then $\norminline{\proj{\set S}(\vect v)}\leq \norminline{\vect v}$. Thus, $0\leq\norminline{\vect v}-\norminline{\map F_s(\vect v)} = \norminline{\vect v}-\norminline{\proj{\set S}(\vect v)}\leq \norminline{\vect v-\proj{\set S}(\vect v)}= \norminline{\perproj{\set S}(\vect v)},~\vect v\in \Rn{m}$. Substituting $\vect v = \ivect{a}{i}-\ivect{a}{j},~i,j\in[n]$, yields $0\leq\norminline{\ivect{a}{i}-\ivect{a}{j}}-\norminline{\map F_s(\ivect{a}{i})-\map F_s(\ivect{a}{j})}\leq\norminline{\ivect{a}{i}-\ivect{a}{j}-\map F_s(\ivect{a}{i})+\map F_s(\ivect{a}{j})}  \leq \norminline{\perproj{\set S}(\ivect{a}{i})}+\norminline{\perproj{\set S}(\ivect{a}{j})}\leq 2\mu$. The last inequality is due to Eq.~\ref{eq:cond1}.
\end{proof}

We stress the fact that our goal is to approximate the \emph{geometry} of the dataset $\set A$ rather than its members, therefore, we use $\map F_s$ rather than $\proj{\set S}$.

\subsection{PCA-based dimensionality reduction}\label{sec: svd}

A common practice to achieve dimensionality reduction is based on PCA~\cite{hotelling:PCA} of the (centered) $m\times n$ data matrix $A$. This method uses a singular value decomposition (SVD)~\cite{golub} of the data matrix to detect a set of maximum variance orthogonal directions (singular vectors) in $\Rn{m}$. Projection of the data onto the $\rho$ most significant directions yields the best $\rho$-dimensional embedding of the data in the mean square error sense. The  computational and storage complexities of SVD are $O(\min\{m,n\}\cdot mn)$ and $O(\min\{m,n\}^2)$, respectively.

Numerical methods for the computation of SVD approximation have attracted a growing interest. Out of many methods, we mention here some central ones. In recent years, randomized algorithms for SVD approximation of large matrices have become popular. We refer the reader to~\cite{halko2011finding} and references therein for a review of such methods. An efficient incremental algorithms for computing a thin SVD that considers $\rho$ components is suggested in~\cite{Brand2006}. The computational complexity of this method is $O(\rho nm)$ for $\rho \leq \min\{m,n\}$. The incremental nature of the algorithm makes it suitable for analysis of dynamic data, where rows/columns are dynamically added and subtracted from the data matrix. Another interesting approach, which  reduces the SVD computational cost, is to use  matrix sparsification by zeroing out small values in the data matrix. This widely used approach utilizes a sparse eigensolver such as Lanczos to compute the relevant $\rho$ eigen-components~\cite{Lanczos}. If $\rho$ is small in comparison to the matrix size, then the computational complexity of Lanczos is $O \left(\max\{m,n\}^2 \cdot\rho \right)$~\cite{golub}. The storage requirements for Lanczos is $O(m)$. Additionally, Lanczos method can be modified to terminate when the smallest estimated  eigenvalue is well approximated and its value is lower than a given threshold. More sparsification approaches are given in~\cite{vonLuxburg}. Finally, the Nystr\"om extension method~\cite{baker:nystrum} provides an additional technique to reduce the SVD computation cost by using a low rank sketch of the data matrix.

Mathematically, suppose that the rank of $A$ is $\rho$. Let $A=USV^\ast$ be the (thin) SVD of $A$, where $U$ and $V$ are $m\times \rho$ and $n\times \rho$ matrices, respectively, whose columns are orthonormal, and $S$ is a diagonal $\rho\times \rho$ matrix, whose diagonal elements are ordered decreasingly $\iscal{s}{1}\geq\ldots\geq \iscal{s}{\rho}\geq 0$. The columns of $U$ and $V$ are referred to as the left and right singular vectors of $A$ respectively, and the diagonal elements of $S$ as its singular values. Then, for any $k \in [\rho]$ we have $\norm{A-A_k} \leq\norm{A-B} $ for any orthogonally invariant matrix norm and any $m\times n$ matrix $B$ of rank $k$ or less, where $A_k$ is the $k$-SVD of $A$, i.e. $A_k = \cols{U}{[k]}\mentry{S}{[k]}{[k]}(\cols{V}{[k]})^\ast$. Let $\set U$ be the subspace spanned by $\cols{U}{[k]}$'s columns, then the $k$-dimensional embedding $\map F_k:\Rn{m}\to\Rn{k}$, which is defined by $\map F_k(\vect v) \triangleq (\cols{U}{[k]})^\ast\vect v$, is a composition of the orthogonal map $\map O_{\set U}:\Rn{m}\to\Rn{k}$, $\map O_{\set U}(\vect v)\triangleq (\cols{U}{[k]})^\ast\vect v$  and the orthogonal projection $\proj{\set U}$, i.e. $\map F_k = \map O_{\set U}\circ\proj{\set U} $ (see Eq.~\ref{eq:Fdecomp}). Lemma~\ref{lem: spect_dist} quantifies the distortion rate of $\map F_k$, applied to $\set A$, with respect to the spectrum of $A$, which is encapsulated in $S$.

\begin{lemma}
\label{lem: spect_dist}
The $k$-dimensional embedding $\map F_k$ is a $2\iscal{s}{k+1}$-distortion of $\set A$.
\end{lemma}

\begin{proof}
From the triangular inequality we have $\norminline{\cols{A}{i}-\cols{A}{j}} \leq \norminline{\perproj{\set U}(\cols{A}{i})} +\norminline{\perproj{\set U}(\cols{A}{j})} +\norminline{\proj{\set U}(\cols{A}{i})-\proj{\set U}(\cols{A}{j})} $. Since $\perproj{\set U}=I-\proj{\set U}$, and due to $A$'s SVD, we have \textbf{  $\norminline{\proj{\set U}(\cols{A}{i})}<\iscal{s}{k+1}$ and $\norminline{\proj{\set U}(\cols{A}{j})}<\iscal{s}{k+1}$.} Moreover, since $\map F_k$ is an orthogonal map, we have $\norminline{\map F_k(\cols{A}{i}-\cols{A}{j})} \leq \norminline{\cols{A}{i}-\cols{A}{j}} $. Thus, $\absinline{\norminline{\cols{A}{i}-\cols{A}{j}}-\norminline{\map F_k(\cols{A}{i}-\cols{A}{j})}}\leq 2\iscal{s}{k+1}$.
\end{proof}
A particular case of Lemma~\ref{lem: spect_dist} is when $k=\rho$. Then, $\map F_k$ embeds $\set A$ accurately in $\Rn{\rho}$. The computational and storage complexities of the thin SVD are $O(\rho mn)$ and $O(\max\{m,n\}^2)$, respectively. In addition, the principal subspace $\set U$, on which the data is projected, is a mixture of the entire columns set of $A$ which, in terms of data analysis, may be less informative than a dictionary-based subspace.

\section{Incomplete Pivoted QR-based Data Analysis}\label{sec:qrdr}
In this section, a QR-based method for data sampling and dimensionality reduction  is suggested, as well as consequent out-of-sample and anomaly detection schemes. The method is geometrically driven in the sense that a low-dimensional approximation is constructed to constitute a user-defined distortion of the high-dimensional dataset $\set A$. This is accomplished by using an incomplete pivoted QR decomposition of the data matrix $A$ that is described in section~\ref{sec:PIQR}.

The suggested method incrementally and simultaneously constructs a low-dimensional subspace and projects the data on it. The basis elements for the constructed subspace are chosen from $\set A$. Thus, this method also identifies a subset of representative landmark data points according to the user-defined distortion parameter. The landmarks subset $\set{D}\subset\set{A}$ is referred to as \emph{dictionary}. The dictionary enables both an efficient out-of-sample extension and anomaly detection for any data point $\vect x\in\Rn{m}\backslash\set{A}$. In this context, $\set A$ is referred as a \emph{training set}, and each member is considered as normal. The out-of-sample extension is based only on the geometrical relations between $\vect x$ and the dictionary members as described in Section~\ref{sec: oose}.

Both the computational and the storage costs of the proposed method depend on the dimension of the embedded space. In the worst case, where the dictionary consists of the whole data, these complexities are identical to the corresponding complexities of the thin SVD of the associated data matrix $A$. Moreover, since the proposed algorithm neither uses the powers of $AA^\ast$ nor $A^\ast A$, as opposed to classical algorithms for SVD computations~\cite{golub}, there is no necessity to store $A$ in the RAM.

There are several methods for practical computation of QR decomposition. Householder~\cite{HouseholderQR1958}, Givenes rotations~\cite{GivensQR1958} and  Gram-Schmidt or modified Gram-Schmidt~\cite{GS1966} are some typical methods. In~\cite{Jennings87}, an incomplete Gram-Schmidt and incomplete Givens transform are utilized to find an incomplete QR decomposition. Another relevant approach is the Rank Revealing QR (RRQR) method~\cite{Chan1987}. The RRQR can be used for matrix approximation by proper manipulation of the QR output~\cite{Chan1992}. The proposed pivoted incomplete QR algorithm is one of many methods to compute a partial orthogonal decomposition~\cite{Bai2001,Papadopoulos02aclass}. Yet, the theoretical basis for our method is valid for any other version of pivoted incomplete QR algorithm.

Robustness of the proposed method to noise is presented in Section~\ref{sec:noise} and the resulted matrix approximation is proved in Section~\ref{sec:mat_appx}.

\subsection{Mathematical preliminaries}\label{sec:QR}

QR factorization with columns pivoting~\cite{golub} of an $m\times n$ matrix $A$ of rank $\rho$ is
\begin{equation}
\label{eq:pqr}
A\Pi=QR,
\end{equation}
where $\Pi$ is an $n\times n$ permutation matrix, $Q$ is an $m\times \rho$ matrix whose columns constitute an orthonormal basis for the columns space of $A$ and $R$ is a $\rho\times n$ upper diagonal matrix. This decomposition represents the Gram-Schmidt process applied to $A$'s columns one-by-one due the order determined by $\Pi$. Therefore, for any $k\in[\rho]$ we have
\begin{equation}
\label{eq:W_i}
\Wpik{k} = \set Q_{[k]},
\end{equation}
and
\begin{equation}
\label{eq:Q}
\cols{Q}{k} = \perpiproj{k-1}(\cols{A}{\mapi(k)})/\norm{\perpiproj{k-1}(\cols{A}{\mapi(k)})},
\end{equation}
where $\set A_\set I$ and $\set Q_{\set I}$ are the subspace spanned by the columns of $\cols{A}{\set I}$ and $\cols{Q}{\set I}$, respectively, and $\piproj{k}:\Rn{m}\to\Wpik{k}$ is the orthogonal projection on $\Wpik{k}$. Equation~\ref{eq:W_i} suggests that for any $k\in[\rho]$
\begin{equation}
\label{eq:proj_rep}
\piproj{k}(\vect v) = \cols{Q}{[k]}(\cols{Q}{[k]})^\ast\vect v,~\perpiproj{k}(\vect v) = \vect v-\piproj{k}(\vect v),\quad\vect v\in\Rn{m}.
\end{equation}

The presented dimensionality reduction method is based on the incomplete pivoted QR decomposition of the data matrix $A$. The criteria for pivoting and incompleteness are based on Lemma~\ref{lem:rii}:
\begin{lemma}
\label{lem:rii}
Consider Eq.~\ref{eq:pqr}. Then, for any $k\in[\rho]$,
\begin{equation}\label{eq:r_ii}
\mentry{R}{k}{k} = \norm{\perpiproj{k-1}(\cols{A}{\mapi(k)})}.
\end{equation}
\end{lemma}
\begin{proof}
Since $\perpiproj{k-1}$ is an orthogonal projection, then $(\perpiproj{k-1})^\ast \perpiproj{k-1} = \perpiproj{k-1}$. Therefore,
\begin{eqnarray*}
\norm{\perpiproj{k-1}(\cols{A}{\mapi(k)})}^2 & = & (\perpiproj{k-1}(\cols{A}{\mapi(k)}))^\ast \perpiproj{k-1}(\cols{A}{\mapi(k)})\\
& = & (\cols{A}{\mapi(k)}))^\ast \perpiproj{k-1}(\cols{A}{\mapi(k)})\\
& = &\norm{\perpiproj{k-1}(\cols{A}{\mapi(k)})}(\cols{A}{\mapi(k)}))^\ast \cols{Q}{[k]}\\
& = &\norm{\perpiproj{k-1}(\cols{A}{\mapi(k)})}\mentry R{k}{k}.
\end{eqnarray*}
\end{proof}
Lemma~\ref{lem: q_rec} stresses the recursive relations between $Q$'s columns. This relation will be used in  Section~\ref{sec: oose}.
\begin{lemma}\label{lem: q_rec}
For any $k\in[\rho]$, $\cols{Q}{k} = (\cols{A}{\mapi(k)}-\sum_{i=1}^{k-1}\mentry{R}{i}{k} \cols{Q}{i})/\mentry{R}{k}{k}.$
\end{lemma}
\begin{proof}
According to Eqs.~\ref{eq:Q} and~\ref{eq:proj_rep}, and since $\mentry{R}{i}{j} = (\cols{Q}{i})^\ast\cols{A}{\mapi(j)}$ for any $i\in[\rho], j\in[n]$, we have $\map W_{\mapi([k-1])}(\cols{A}{\mapi(k)}) = \sum_{i=1}^{k-1}\mentry{R}{i}{k} \cols{Q}{i}$. Therefore, due to Eq.~\ref{eq:r_ii}, the lemma is proved.

\end{proof}

\subsection{Incomplete pivoted QR-based dimensionality reduction - \\theoretical background}\label{sec:dr}
The geometry of $A$'s (permuted) columns is isomorphic to the geometry of the the $R$'s columns, i.e. for any $i,j\in[n]$, $(\cols{A}{\mapi(i)})^\ast \cols{A}{\mapi(j)} = (\cols{R}{i})^\ast \cols{R}{j}$. Thus, the upper triangularity of $R$ suggests to embed the dataset $\set A$ by an incomplete (truncated) version of  $R$'s rows. Mathematically, following Eq.~\ref{eq:Fdecomp}, if we set $\set S=\Wpik{s}$ and the orthogonal map $\map O_{\set S}:\Rn{m}\to\Rn{s}$ is
\begin{equation}
\label{eq:O}
\map O_{\set S}(\vect v)\triangleq (\cols{Q}{[s]})^\ast\vect v,
\end{equation}
then, due to the orthogonality of $Q$'s columns and Eqs.~\ref{eq:pqr},~\ref{eq:W_i} and~\ref{eq:proj_rep}, the $s$-dimensional embedding from Eq.~\ref{eq:Fdecomp} becomes \begin{equation}
\label{eq:compact_F_s}
\map F_s(\vect v) = (\cols{Q}{[s]})^\ast \vect v
\end{equation}
and specifically,
\begin{equation}
\label{eq:embed}
\map F_s(\cols{A}{\mapi(i)}) = \mentry{R}{[s]}{i},\quad i\in[n].
\end{equation}
Notice that Eq.~\ref{eq:cond2} is satisfied by $\map O_{\set S}$ from Eq.~\ref{eq:O}. Although this specific choice for $\map O_\set S$ yields $\map O_\set S = \map F_s$, this is not always the case since, as aforementioned, $\map O_\set S$ is not unique. For example, in Section~\ref{sec:PIQR}, a different choice of $\map O_\set S$ is presented. The incompleteness of the discussed QR decomposition is reflected in Eq.~\ref{eq:embed}, where the $s$-dimensional embedding is defined via only a partial set of $R$'s rows. According to the triangularity of $R$, the geometry of such an embedding is exact on the basis elements of $\set S$ which are
\begin{equation}
\label{eq:dict}
\set D\triangleq\{\cols{A}{\mapi(1)},\ldots,\cols{A}{\mapi(s)}\}.
\end{equation}
This set is referred to as the \emph{dictionary} of $\set A$ and its elements are referred to as \emph{pivots}. As Lemma~\ref{prop:fund} suggests, a careful choice of $\Pi$ might result in a low-dimensional distortion $\map F_s$ of $\set A$. An algorithm for such a choice is presented in Section~\ref{sec:PIQR}.

We conclude this section with an example that demonstrates the permutation's significance. Let $A$ be the following matrix:
\begin{equation*}
A = \left( \begin{array}{ccccccc}
1 & 1 & 1 & 1 & 1 & 1 & 1\\
0 & 1 & 1 & 1 & 1 & 1 & 1\\
0 & 0 & 1 & 1 & 1 & 1 & 1\\
0 & 0 & 0 & 1 & 1 & 1 & 1\\
0 & 0 & 0 & 0 & 1 & 1 & 1\\
0 & 0 & 0 & 0 & 0 & 1 & 1\\
0 & 0 & 0 & 0 & 0 & 0 & 20
\end{array} \right).
\end{equation*}
Then, in order to achieve a $3$-distortion for $\Pi = I_7$, we must have $s=\rho=7$. On the other hand, one can verify that for any permutation that satisfies $\mapi([3])=\{1,4,7\}$, $\norminline{\perproj{\mapi([3])}(\cols{A}{i})} < 1.5,~i\in[7]$, which according to Lemma~\ref{prop:fund} is a sufficient condition for a $3$-dimensional embedding of $\set A$, with distortion rate bounded by $3$. Consequently, $F_s$ is an $s$-dimensional $2\mu$-distortion of $\set A$, with $s=3$ and $\mu=1.5$.

Proposition~\ref{prop:main}, which is a rephrased version of Lemma~\ref{prop:fund} in terms of the pivoted incomplete QR, concludes the above discussion:
\begin{proposition}
\label{prop:main}
Let $\mu>0$. If there exist a permutation $\mapi$ and $s\in[\rho]$ for which $\norminline{\perpiproj{s}(\cols{A}{i})} <\mu$ for any $i\in[n]$, then $\map F_s$ from Eq.~\ref{eq:compact_F_s} is an $s$-dimensional $2\mu$-distortion of $\set A$.
\end{proposition}

\subsubsection{Stability to noise}\label{sec:noise}

In real-life, data may be noisy. Therefore, instead of analyzing clean data that is stored in $A$, a noisy version $\tilde A = A + N$ is analyzed where the matrix $N$ represents an additive noise. Proposition~\ref{prop:noise} shows that a distortion by a noisy data is also a distortion of the original clean data, where the error originated by noise is additive.

 \begin{proposition}\label{prop:noise}
Let $A$ be an $m\times n$ data matrix and let $\tilde A = A+N$, where $N$ is a noise matrix of the same size of $A$. Assume that $\norminline{N}\leq\eta$ for some $\eta\geq 0$, and let $\tilde{\map F}_s$ be a $2\mu$-distortion of  $\tilde A$'s columns as defined in Eq.~\ref{eq:embed}. Then, $\tilde{\map F}_s$ is a $2(\mu+\eta)$-distortion of  $A$'s columns.
\end{proposition}
\begin{proof}
Let $\tilde{\map F}_s$ be the map defined in Eq.~\ref{eq:Fdecomp} with the corresponding elements $\tilde{\set S}$, ${\map O}_{\tilde{\set S}}$  and $\proj{\tilde{\set S}}$. Then, we have \begin{eqnarray*}
\abs{\norm{\cols{ A}{i}-{\cols{ A}{j}}}-\norm{\tilde{\map F}_s(\cols{\tilde A}{i})-\tilde{\map F}_s(\cols{\tilde A}{j})}} & = & \abs{\norm{\cols{A}{i}-{\cols{A}{j}}}-\norm{\proj{\tilde {\set S}}(\cols{\tilde A}{i})-\proj{\tilde{\set S}}(\cols{\tilde A}{j})}}\\
& \leq & \norm{\cols{A}{i}-{\cols{A}{j}}-\proj{\tilde {\set S}}(\cols{\tilde A}{i})+\proj{\tilde{\set S}}(\cols{\tilde A}{j})}\\
& \leq & \norm{\cols{A}{i}-\proj{\tilde {\set S}}(\cols{\tilde A}{i})} + \norm{\cols{A}{j}-\proj{\tilde{\set S}}(\cols{\tilde A}{j})}\\
& \leq & \norm{\cols{A}{i}-\cols{\tilde A}{i}} + \norm{\cols{\tilde A}{i}-\proj{\tilde {\set S}}(\cols{\tilde A}{i})}\\
& + & \norm{\cols{A}{j}-\cols{\tilde A}{j}} + \norm{\cols{\tilde A}{j}-\proj{\tilde {\set S}}(\cols{\tilde A}{j})}\\
& = & \norm{\cols{N}{i}} + \norm{\cols{\tilde A}{i}-\proj{\tilde {\set S}}(\cols{\tilde A}{i})}\\
& + & \norm{\cols{N}{j}} + \norm{\cols{\tilde A}{j}-\proj{\tilde {\set S}}(\cols{\tilde A}{j})}\\
& \leq & 2(\mu+\eta).
\end{eqnarray*}
The first equality is due to Eqs.~\ref{eq:cond2} and~\ref{eq:Fdecomp} and the last inequality is due to the fact that $\norminline{\cols{N}{i}}\leq\norminline{N}$, the proposition's assumption and Eq.~\ref{eq:cond1}.

\end{proof}
\subsubsection{Matrix approximation error}\label{sec:mat_appx}
The notions of low-rank matrix approximation and geometry-preserving dimensionality reduction are different but related. While the $s$-SVD of a data matrix enables an $s$-dimensional embedding of the associated data (as was established in Lemma~\ref{lem: spect_dist}), it can be shown that the incomplete QR factorization $\cols{Q}{[s]}\rows{R}{[s]}$ of the pivoted version of $A$ (see Eq.~\ref{eq:pqr}) can be used to form a low-rank approximation of $A$. The operator norm of an arbitrary matrix $M$, denoted by $\norminline{M}_{2}$, is equal to its maximal singular value. Thus, as was explained in Section~\ref{sec: svd}, it measures the maximal linear trend of the data stored in its columns (or rows). Therefore, the operator norm of the difference of two matrices measures the strength of the maximal linear trend of the associated error. Proposition~\ref{prop: approx} provides a bound for the approximation error of $A\Pi$ by its incomplete QR factorization. Its proof uses the Frobenius norm $\norminline{M}_{F}\triangleq \sqrt{ \sum_{i,j}\mentry{M}{i}{j}^2}$ of a matrix and the norms inequality $\norminline{M}_{2}\leq\norminline{M}_{F}$ for any matrix $M$.
\begin{proposition}\label{prop: approx}
Let $\mu$, $\mapi$, $s$ and $A$ satisfy the condition of Proposition~\ref{prop:main},  then $\norminline{A\Pi - \cols{Q}{[s]}\rows{R}{[s]}}_{\eta}\leq\mu\sqrt{\rho-s}$ for $\eta\in\{2,F\}$.
\end{proposition}

\begin{proof}
Following Eqs.~\ref{eq:pqr} and~\ref{eq:proj_rep}, and according to the orthonormality of $Q$'s columns, we have $\piproj{s}(\cols{A}{i})=\cols{Q}{[s]}\mentry{R}{[s]}{i}$ for any $i\in[n]$. On the other hand, due to the triangularity of $R$, $\cols{A}{\mapi([s])} = \cols{Q}{[s]}\mentry{R}{[s]}{[s]}$. Thus, at least $s$ columns from $A-\cols{Q}{[s]}\cols{R}{[s]}$ are vanishing, and the norms of the rest (mostly) $\rho-s$ are bounded by $\mu$, according to the proposition's assumption. This leads to $\norm{A-\cols{Q}{[s]}\cols{R}{[s]}}_{2}\leq\norm{A-\cols{Q}{[s]}\cols{R}{[s]}}_{F}\leq \mu\sqrt{\rho-s}$.
\end{proof}

\subsection{Incomplete pivoted QR-based (ICPQR) dimensionality reduction: implementation}\label{sec:PIQR}

Algorithm~\ref{alg: tpcd} iteratively constructs an incomplete pivoted $QR$ version of the data matrix $A$ to obtain a $2\mu$-distortion of $\set A$. In its $j$-th iteration, the algorithm selects the  pivot $\cols{A}{\mapi(j)}$, and projects the dataset $\set A$ on $\Wpik{j}$. Based on Step~1 from Section~\ref{sec:genappch}, the $j$-th pivot $\cols{A}{\mapi(j)}$ is chosen to be the element in $\set A$, whose approximation by its orthogonal projection on $\Wpik{j-1}$ is the worst. Thus, the permutation $\mapi:[n]\to[n]$ is determined by the following condition:
\begin{equation}
\label{eq:pi}
\mapi(j) = \arg\max_{i\in[n]\backslash\mapi([j-1])}\norm{\perpiproj{j-1}(\cols{A}{i})}.
\end{equation}
Then, the corresponding new column $\cols{Q}{j}$ and row $\rows{R}{j}$ are computed according to Eqs.~\ref{eq:pqr}, \ref{eq:Q} and~\ref{eq:proj_rep}. Since the columns permutation is updated in every iteration $j$, the columns of $\cols{R}{[j-1]}$ have to be permuted correspondingly. The algorithm terminates when the quantity in Eq.~\ref{eq:pi} is less than $\mu$. Thus $s$, which is the number of iterations required to provide a $2\mu$-distortion, is a non-increasing function of $\mu$, bounded from above by $\rho$ that is not known a-priori. When the algorithm ends, the correspondence rule, defined by Eq.~\ref{eq:embed}, provides an $s$-dimensional $2\mu$-distortion of $\set A$ according to Proposition~\ref{prop:main}. For $\mu=0$, the application of Algorithm~\ref{alg: tpcd} to $A$ results in a  complete pivoted QR factorization.

\IncMargin{1em}
\begin{algorithm}[h!]
\DontPrintSemicolon
\SetAlgoLined
\SetKwComment{tcp}{//}{}
\SetKwInOut{Input}{Input}\SetKwInOut{Output}{Output}
\Input{An $m\times n$ data matrix $A$ and a nonnegative distortion parameter $\mu$.}
\Output{An $m\times s$ matrix $Q$ whose columns are orthonormal, an $s\times n$ upper diagonal matrix $R$ and a permutation $\mapi$ such that the correspondence rule  defined by Eq.~\ref{eq:embed} is a $2\mu$-distortion of $A$'s columns.} \BlankLine
Initialization: set $\mapi = identity$, $\Pi = I_n$, $\delta > \mu$, and $j=0$ \\
\While{$\delta > \mu$}
{
    set $j=j+1$\\
    set $i_j = \arg\max_{i\in[n]\backslash\mapi([j-1])}\norminline{\perpiproj{j-1}(\cols{A}{i})} $ (see Eq.~\ref{eq:proj_rep})\\
    \tcp{consider $\mapi([0])=\emptyset$}\label{alg1:stp_col_choose}
    set $\vect \Delta = \perpiproj{j-1}(\cols{A}{i_j})$ and $\delta = \norm{\vect\Delta} $\\
    set $\cols{Q}{j} = \vect \Delta/\delta$\\
    switch $\mapi(j) \leftrightarrow \mapi(i_j)$, set $\Pi = \Pi\Pi_{j\leftrightarrow i_j}$ \\
    \tcp{$\Pi_{j\leftrightarrow i_j}$ is $I_n$ with columns $j,i_j$ swapped}
    set $R = R\Pi_{j\leftrightarrow i_j}$\label{alg1:stp_outpts1}\\
    set $\rows{R}{j} = (\cols{Q}{j})^\ast A\Pi$ \label{alg1:stp_outpts2}\\
} set $s = j$
 \caption{Incomplete pivoted QR (ICPQR) decomposition}
\label{alg: tpcd}
\end{algorithm}
\DecMargin{1em}

Let us make a couple of technical remarks concerning Algorithm~\ref{alg: tpcd}: 1. In case of limited computational or storage budget, Algorithm~\ref{alg: tpcd} can be easily modified to make a limited number of iterations $d$ or, equivalently, to provide a $d$-dimensional embedding. In this case, the distortion parameter $\mu$ is a non-increasing function of $d$. 2. The dictionary $\set D$ is chosen regardless to the data indexing order. This property ensures a relatively sparse dictionary as demonstrated in Section~\ref{subsec:cmp_idm}. 3.In order to achieve an optimal\footnote{A coordinates system that is determined by principal components.} coordinates system for the geometry represented by $R$, an SVD can be utilized. Mathematically, let $R=USV^\ast$ be the SVD decomposition of $R$, where $U$ and $S$ are $s\times s$ orthogonal and diagonal matrices, respectively, and $V$ is an $n\times s$ matrix, whose columns are orthonormal. Then, the map $\map{\hat O_{\set S}}:\Rn{m}\to\Rn{s}$,
\begin{equation}
\label{eq:Ohat}
\map{\hat O_{\set S}}(\vect v)\triangleq (QU)^\ast\vect v,\quad\vect v\in\Rn{m}
\end{equation}
is still isometric on $\set S$, as Step~2 in Section~\ref{sec:genappch} requires. Thus, following Eqs.~\ref{eq:pqr},~\ref{eq:embed} and~\ref{eq:Ohat}, the map $\map{\hat  F_s}:\Rn{m}\to\Rn{k}$ where $
\map{\hat  F_s}(\vect v)\triangleq \map{\hat O_{\set S}}\circ\piproj{s}(\vect v),\quad\vect v\in\Rn{m}
$
is still an $s$-dimensional $2\mu$ distortion of $\set A$. Moreover, $\map{\hat  F_s}$ is optimal in the sense that the axes are aligned correspondingly to the variances directions. The computational and storage costs of such an alignment are $O(ns^2)$ and $O(ns)$, respectively. Therefore, the total complexity of Algorithm~\ref{alg: tpcd} is not affected by this optional step (see Table~\ref{tbl:complexity}.)

\subsection{ICPQR reduced cost}\label{ssec:boost}
Equation~\ref{eq:embed} suggests that $Q$ is not needed for the low rank embedding of $\set A$ by $\map F_s$. In this section, we present a more efficient version of Algorithm~\ref{alg: tpcd}, by which the results in Section~\ref{sec:experimental} were obtained. The algorithm produces no $Q$ and applies no physical permutations.

Consider Eq.~\ref{eq:pqr} with $\mapi$ defined by Eq.~\ref{eq:pi}, then
\begin{equation}
\label{eq:prmtdQR}
A = Q\bar R,\quad \bar R\triangleq R\Pi^\ast,
\end{equation}
where $\bar R$ is no longer triangular. Algorithm~\ref{alg: q-less} is a translated version of Algorithm~\ref{alg: tpcd} to this case, where the permutation $\Pi$ is absorbed in $\bar R$. The low-dimensional embedding from Eq.~\ref{eq:embed} becomes
\begin{eqnarray}~\label{eq:embed2}
\map F_s(\cols{A}{i}) = \mentry{\bar R}{[s]}{i},\quad i\in[n].
\end{eqnarray}
For the establishment of Algorithm~\ref{alg: q-less}, steps~\ref{alg1:stp_col_choose} and~\ref{alg1:stp_outpts2} in Algorithm~\ref{alg: tpcd}, which are dependent on $Q$, are modified according to Eq.~\ref{eq:prmtdQR} to be $(\cols{A}{\mapi(i)})^\ast \cols{A}{\mapi(j)} = (\cols{\bar R}{\mapi(i)})^\ast \cols{\bar R}{\mapi(j)}=\sum_{\ell=1}^{\min\{i,j\}}\mentry{\bar R}{\ell}{\mapi(i)}\mentry{\bar R}{\ell}{\mapi(j)}$, for any $i,j\in[n]$. The upper limit in the sum is due to the upper triangularity of $R$. Thus, the following recursive relations between the entries of $\bar R$ hold:
\begin{equation}\label{eq:q_less_r_ij}
\mentry{\bar R}{i}{\mapi(j)} = \left\{\begin{array}{ll} \mentry{\bar R}{i}{\mapi(i)}^{-1}\left(u_{ij}- \sum_{\ell=1}^{i-1} \mentry{\bar R}{\ell}{\mapi(i)}\mentry{\bar R}{\ell}{\mapi(j)}\right) & \textrm{if $i<j$}\\
\left(u_{ij}- \sum_{\ell=1}^{i-1} (\mentry{\bar R}{\ell}{\mapi(i)})^2\right)^{1/2}& \textrm{if $i=j$}\\
0 & \textrm{if $i>j$}
\end{array}\right.,
\end{equation}
where $u_{ij} \triangleq (\cols{A}{\mapi(i)})^\ast \cols{A}{\mapi(j)}$. Moreover, comparing Eq.~\ref{eq:r_ii} with Eq.~\ref{eq:q_less_r_ij} yields
$$
\norm{\perpiproj{i-1}(\cols{A}{\mapi(i)})}  = \left(u_{ii}- \sum_{\ell=1}^{i-1} (\mentry{\bar R}{\ell}{\mapi(i)})^2\right)^{1/2}.
$$
Thus, the pivoting criterion from Eq.~\ref{eq:pi} becomes $$\mapi(j) = \arg\max_{i\in[n]\backslash[j-1]}\left(u_{ii}- \sum_{\ell=1}^{i-1} (\mentry{\bar R}{\ell}{\mapi(i)})^2\right)^{1/2}.$$

\IncMargin{1em}
\begin{algorithm}
    \DontPrintSemicolon
    \SetAlgoLined
    \SetKwComment{tcp}{//}{}
    \SetKwInOut{Input}{Input}\SetKwInOut{Output}{Output}
    \Input{An $m\times n$ matrix $A$ and a nonnegative distortion parameter $\mu$.}
\Output{An $s\times n$ matrix $\bar R$ and a permutation $\mapi$, for which the embedding defined by Eq.~\ref{eq:embed2} is a $2\mu$-distortion of $A$'s columns.}
    \BlankLine
    Initialization: set $\mapi = identity$, $j=0$, $\delta > \mu^2$, $\vect y= \vect 0_n$ (the all zeros vector of length $n$), and $\vect z\in\Rn{n}$, for which $\ventry{z}{i} = \norminline{\cols{A}{i}} ^2$, $i\in[n]$ \label{alg2:init} \\
    \While{ $\delta \geq \mu^2$}
    {
        set $j=j+1$\\
        set $i_j = \arg\max_{i\in[n]\backslash[j-1]}(\ventry{z}{\mapi(i)}- \ventry{y}{\mapi(i)})$ \label{alg2:max_arg}\\
        switch $\mapi(j) \leftrightarrow \mapi(i_j)$     \label{alg2:switch}\\
        set $\delta =  \ventry{z}{\mapi(j)}- \ventry{y}{\mapi(j)} $\label{alg2:delta}\\
        for every $i\in[j-1]$ set $\mentry{\bar R}{j}{\mapi(i)} = 0$  \label{alg2:under_diag}\\
        set $\mentry{\bar R}{j}{\mapi(j)} = \delta^{1/2}$    \label{alg2:diag}\\
        \For {$i\in [n]\backslash[j]$} 
        {
            \quad set $u_{ji} = (\cols{A}{\mapi(j)})^\ast \cols{A}{(\mapi(i))}$ \label{alg2:above_diag1}\\
            \quad $\mentry{\bar R}{j}{\mapi(i)} = (u_{ji}-\sum_{\ell=1}^{j-1}\mentry{\bar R}{\ell}{\mapi(j)} \mentry{\bar R}{\ell}{\mapi(i)})/\mentry{\bar R}{j}{\mapi(j)}$\label{alg2:above_diag2}\\
            \quad set $\ventry{y}{\mapi(i)} = \ventry{y}{\mapi(i)}+\mentry{\bar R}{j}{\mapi(i)}$    \label{alg2:above_diag3}\\
        }

    }
    set $s = j$
    \caption{Incomplete pivoted Q-less QR decomposition}
    \label{alg: q-less}
\end{algorithm}
\DecMargin{1em} The resulted dictionary is the set $\set D$ as defined in Eq.~\ref{eq:dict}.

Table~\ref{tbl:complexity} presents the computational and storage complexities of Algorithm~\ref{alg: q-less}. The storage of the input matrix $A$ was not taken into account since in the nature of Algorithm~\ref{alg: q-less} there is no need to have its complete storage.  For example, the relevant rows and columns of $A$ can be individually computed at each iteration. Therefore, the total storage complexity is smaller than the required storage of the SVD, which is $O(\max\{m,n\}^2)$. The computational complexity of Algorithm~\ref{alg: q-less} depends on $\mu$. In the worst case, when $\mu= 0$ and  $s=\rho$, the complexity of Algorithm~\ref{alg: q-less} equals to the complexity of the thin rank-$\rho$ SVD. Otherwise, Algorithm~\ref{alg: q-less} is more efficient than SVD.
\begin{table}[ht]
    \centering
    \begin{tabular}{|l|l|l|}
\hline Step & Operations & Storage (not including storage of $A$) \\
        \hline  \ref{alg2:init} & $O(mn)$ & $O(n)$ \\
        \hline  \ref{alg2:max_arg}& $O(ns-s^2)$ & $O(1)$ \\
        \hline  \ref{alg2:under_diag}& $O(s^2)$ & $O(s^2)$ \\
        \hline  \ref{alg2:above_diag1}& $O(m(ns-s^2))$ & $O(m)$\\
        \hline  \ref{alg2:above_diag2}& $O(s(ns-s^2))$ & $O(ns-s^2)$\\
        \hline  \ref{alg2:above_diag3}& $O(ns-s^2)$ & $O(1)$\\
        \hline Total: & $O(mns) $ & $ O(ns) $\\ \hline
    \end{tabular}
    \caption{Computational and storage complexities of Algorithm~\ref{alg: q-less}.}
\end{table}\label{tbl:complexity}

\subsection{Out-of-sample extension and anomaly detection algorithms}\label{sec: oose}
Two fundamental questions  may be naturally asked for an out-of-sample data point $\vect x\in\Rn{m}\backslash\set A$: first, is it normal related to the training dataset $\set A$? and secondly, if it is, then how can the produced low rank embedding be extended to this data point? This section addresses these two questions. Since $\map F_s$ (Eq.~\ref{eq:compact_F_s}) is defined for the entire space of $\Rn{m}$, it is used to define an out-of-sample extension for the embedding from Eq.~\ref{eq:embed2} of  any $\vect x\in\Rn{m}$ and, based on this, to detect anomalies.

\subsubsection{Out-of-sample extension}\label{sec:oose}

As mentioned above, an out-of-sample extension of the embedding from Eq.~\ref{eq:embed2} is defined by Eq.~\ref{eq:compact_F_s} for the entire $\Rn{m}$. As discussed in Section~\ref{ssec:boost}, since Algorithm~\ref{alg: q-less} produces no $Q$, $\map F_s$ cannot be directly applied to an out-of-sample point $\vect x\in \Rn{m}$. Therefore, a Q-less tool for calculating the out-of-sample extension, as defined in Eq.~\ref{eq:compact_F_s}, is provided in Algorithm~\ref{alg: q-less_oos} that is based on the following proposition:
\begin{proposition}\label{prop:q_less_oose}
Let $\bar R$ be the $s\times n$ matrix produced by Algorithm~\ref{alg: q-less}, $\vect x\in\Rn{m}$ and $\vect f = \map F_s(\vect x)\in\Rn{s}$ as defined in Eq.~\ref{eq:compact_F_s}. Then, the following recursive relation holds for the coordinates of $\vect f$: $$\ventry{f}{j} = (\mentry{\bar R}{j}{\mapi(j)})^{-1}(\cols{A}{\mapi(j)})^\ast\vect x - \sum_{i=1}^{j-1} \mentry{\bar R}{i}{\mapi(j)} \ventry{f}{i},\qquad j\in[s].$$
\end{proposition}

\begin{proof}
According to Eq.~\ref{eq:compact_F_s}, $\ventry{f}{j} = (\cols{Q}{j})^\ast \vect x$. Thus, due to the recursive relations of $Q$'s columns, as presented in Lemma~\ref{lem: q_rec}, we have \begin{eqnarray*}
\ventry{f}{j} & = & (\mentry{R}{j}{j})^{-1}((\cols{A}{\mapi(j)})^\ast\vect x - \sum_{i=1}^{j-1}\mentry{R}{i}{j} (\cols{Q}{i})^\ast\vect x)\\
& = & (\mentry{R}{j}{j})^{-1}((\cols{A}{\mapi(j)})^\ast\vect x - \sum_{i=1}^{j-1}\mentry{R}{i}{j} \ventry{f}{i})\\
& = & (\mentry{\bar R}{j}{\mapi(j)})^{-1}((\cols{A}{\mapi(j)})^\ast\vect x - \sum_{i=1}^{j-1}\mentry{\bar R}{i}{\mapi(j)} \ventry{f}{i}),
\end{eqnarray*}
where the last equality is due to Eq.~\ref{eq:prmtdQR}.
\end{proof}

Since the out-of-sample extension of a data point $\vect x\in\Rn{m}$ is its orthogonal projection on the $s$-dimensional dictionary subspace $\set S$, the only required information are the geometric relations between $\vect x$ and the elements of the dictionary $\set D$ (Eq.~\ref{eq:dict}), as Proposition~\ref{prop:q_less_oose} shows. Algorithm~\ref{alg: q-less_oos} summarized the above.
\IncMargin{1em}
\begin{algorithm}
    \DontPrintSemicolon
    \SetAlgoLined
    \SetKwComment{tcp}{//}{}
    \SetKwInOut{Input}{Input}\SetKwInOut{Output}{Output}
    \Input{Dictionary $\set D = \{\ivect{b}{1},\ldots,\ivect{b}{s}\}$ (see Eq.~\ref{eq:dict}) and  $s\times n$ matrix $\bar R$, which are the outputs of Algorithm~\ref{alg: q-less},
     and vector $\vect v\in\Rn{m}$.}
\Output{$\map F_s(\vect v)$ as defined by Eq.~\ref{eq:compact_F_s}.}
    \BlankLine
    Initialization: set $\vect f\in\Rn{s}$ to be vector of all-zeros, except of the first coordinate, $\ventry{f}{1} = (\mentry{\bar R}{1}{\mapi(1)})^{-1}(\ivect{b}{1})^\ast\cdot\vect v$.\\
    \For{ $j=2:s$}
    {
        set the $j$-th coordinate of $\vect f$ to be $\ventry{f}{j} = (\mentry{\bar R}{j}{\mapi(j)})^{-1}(\ivect{b}{j})^\ast\cdot\vect v - \sum_{i=1}^{j-1}\mentry{\bar R}{i}{\mapi(j)} \ventry{f}{i}$\\
    }
    set $\map F_s(\vect v)=\vect f$
    \caption{Out-of-sample extension for incomplete pivoted Q-less QR decomposition}
    \label{alg: q-less_oos}
\end{algorithm}
\DecMargin{1em}

\subsubsection{Anomaly detection}\label{sec:an_det}

The $2\mu$-embedding subspace $\set S\subset \Rn{m}$ satisfies $\bmu(\vect a)\leq\mu$ for any $\vect a\in \set A$, where the distortion rate function  $\boldsymbol\mu:\Rn{m}\to\Rn{}$ is defined to be $\bmu(\vect x)\triangleq \norm{\perproj{\set S}(\vect x)}$ (see discussion in Section~\ref{sec:genappch}). Once the out-of-sample extension of $\map F_s(\vect x)$ was computed by Algorithm~\ref{alg: q-less_oos}, the distortion rate of $\vect x\in\Rn{m}$ can be easily calculated by $\boldsymbol\mu^2(\vect x) = \norm{\vect x} ^2-\norm{\proj{\set S}(\vect x)} ^2 = \norm{\vect x} ^2-\norm{\map O_{\set S} \proj{\set S}(\vect x)} ^2 = \norm{\vect x} ^2-\norm{\map F_s(\vect x)} ^2$. The first equality is due to the fact that $\proj{\set S}$ is orthogonal projection. The second is due to Step~2 in Section~\ref{sec:genappch} that suggests that $\norminline{\map O_{\set S}\circ\proj{\set S}(\vect x)}  = \norminline{\proj{\set S}(\vect x)} $ for any $\vect x\in\Rn{m}$ and the last equality is due to Eq.~\ref{eq:compact_F_s}. Consequently, we rephrase the distortion rate function to be
\begin{equation}
\label{eq:dist_func}
\bmu(\vect x)\triangleq (\norm{\vect x} ^2-\norm{\map F_s(\vect x)} ^2)^{1/2}.
\end{equation}

Normality of data points by definition \ref{def: normality} will serve us in two main forms in the rest of the paper.

\begin{definition}[$\kappa$-normality]\label{def: normality}
Let $\map F_s:\set A\to\Rn{m}$ be an $s$-dimensional embedding of $\set A\subset\Rn{m}$, computed by Algorithm~\ref{alg: q-less}, and let $\map F_s(\vect x)\in\Rn{s}$ be its extension to $\vect x\in\Rn{m}$, produced by Algorithm~\ref{alg: q-less_oos}. Then, $\vect x$ is classified as a $\kappa$-normal point relatively to $\set A$ if $\bmu(\vect x)\leq\kappa$. Otherwise, it is classified as $\kappa$-abnormal.
\end{definition}
Due to Definition~\ref{def: normality}, all the data points in $\set A$ are $\mu$-normal. Nevertheless, $\mu$ is not necessarily the minimal $\kappa$ for which $\set A$ is classified as $\kappa$-normal, as the most strict $\kappa$ for which $\set A$ is still $\kappa$-normal, is $\kappa=\mu_{strict}$, where
$
\mu_{strict}\triangleq\sup_{\vect a\in\set A}\bmu(\vect a).
$

We conclude this section with a definition of two variants of $\kappa$-normality as was defined in Definition~\ref{def: normality}.
\begin{definition}[Normality and strict normality]\label{def: normality2}
Let $\map F_s:\set A\to\Rn{m}$ be an $s$-dimensional embedding of $\set A\subset\Rn{m}$ computed by Algorithm~\ref{alg: q-less}, with distortion parameter $\mu$, and let $\map F_s(\vect x)\in\Rn{s}$ be its extension to $\vect x\in\Rn{m}$ produced by Algorithm~\ref{alg: q-less_oos}. Then, $\vect x$ is classified as a normal data point relatively to $\set A$ if $\bmu(\vect x)\leq\mu$, and as a strictly normal point if $\bmu(\vect x)\leq\mu_{strict}$. Otherwise, it is classified as a (strictly) abnormal.
\end{definition}
Obviously, all the data points in $\set A$ are strictly-normal and any strictly normal data point is also a normal data  point.

\section{QR-based Diffusion Maps}\label{sec:qrdm}
Although the QR-based dimensionality reduction method, which was presented in Section~\ref{sec:qrdr}, is designated for parametric data analysis, this section presents a utilization of our method for the Diffusion Maps (DM)~\cite{coifman:DM2006}, which is a graph Laplacian based method for analysis of nonparametric data, via exploration of a Markov chain defined on the data. It is mainly utilized for clustering and manifold learning. Typically, application of DM involves a kernel PCA, which is computationally prohibitive for large amount of data.

\subsection{DM framework: overview}

\subsubsection{Diffusion geometry} \label{subsec: dm-geometry}
Let $\set X=\{x_1,\ldots,x_n\}$ be a dataset and let $k:\set X\times \set X\to\Rn{}$ be a symmetric point-wise positive kernel that defines a connected undirected weighted graph over $\set X$. Then, a Markov process over $\set X$ can be defined using $n\times n$ row-stochastic transition probabilities matrix
\begin{equation}
\label{eq:P}
P = D^{-1}K,
\end{equation}
where $\mentry{K}{i}{j}=k(x_i,x_j),~i,j\in[n]$ and $D$ is a diagonal matrix with the diagonal elements  $\mentry{D}{i}{i} = \ventry{d}{i}$, where
$
\ventry{d}{i}\triangleq\sum_{j=1}^n \mentry{K}{i}{j},\quad i\in[n].
$
The vector $\vect d\in\Rn{n}$ is referred to as the \emph{degrees function} or \emph{degrees vector} of the graph. The associated time-homogeneous Markov chain is defined as follows: for any two time points $t,t_0\in \mathbbm N$, $\mathbbm{P}(x(t+t_0) = x_j | x(t_0)=x_i) = \mentry{(P^t)}{i}{j}.$ Assuming that the defined Markov chain is aperiodic (for example,  if there is $x\in\set X$, for which $k(x,x)>0$), then it has a unique stationary distribution $ {\vect {\hat d}}\in\Rn{n}$ which is the steady state of the process, i.e. ${\ventry{\hat d}{j}} = \lim_{t\to\infty}\mentry{(P^t)}{i}{j}$, regardless the initial point $x_i$. This steady state is the probability distribution resulted from $\ell_1$ normalization of the degrees function $\vect d$, i.e.,
\begin{equation}
\label{eq:d_hat}
\vect{\hat d} = \vect d/\norminline{\vect d}_{1}.
\end{equation}
The \emph{diffusion distance} in time $t\in\mathbbm N$ is defined by the metric $\timel{\map D}{t}:\set X\times \set X\to\Rn{}$ such that
\begin{equation}
\label{eq:DM}
\timel{\map D}{t}(x_i,x_j)\triangleq \norm{\rows{(P^t)}{i}-\rows{(P^t)}{j}}_{\ell^2(\vect{\hat d}^{-1})},\quad i,j\in[n].
\end{equation}
By definition, $\rows{(P^t)}{i}$ is the probability distribution over $\set X$ after $t$ time steps, where the initial state is $x_i$. Therefore, the diffusion distance from Eq.~\ref{eq:DM} measures the difference between two propagations along $t$ time steps, one originated in $x_i$ and the other in $x_j$. Weighing the metric by the inverse of the steady state results in ascribing high weight for similar probabilities on rare states and vice versa.

Due to the above interpretation, the diffusion distances are naturally utilized  for multiscale clustering since they uncover the connectivity properties of the graph across time. In~\cite{berard, coifman:DM2006}, it was proved that under some conditions, if $\set{X}$ is sampled from a low intrinsic dimensional manifold then, as $n$ tends to infinity, the Markov chain converges to a diffusion process over that manifold.

\subsubsection{Diffusion maps - low rank representation of the diffusion geometry}
\label{subsubsec: svdBasedDM}

Diffusion maps~\cite{coifman:DM2006} are a family of Euclidean representations of the diffusion geometry of $\set X$ in different time steps, where the Euclidean distances approximate the diffusion distances in Eq.~\ref{eq:DM}. Let $\timel{G}{t}$ be the $n\times n$ matrix defined by
\begin{equation}
\label{eq:G}
\timel{G}{t}\triangleq \norm{\vect d}_{1}^{1/2}D^{-1/2}(P^\ast)^t,~t\in\mathbb N.
\end{equation}
Then, due to Eqs.~\ref{eq:P}-\ref{eq:DM}, the Euclidean $n$-dimensional geometry of  $\timel{G}{t}$'s column is isomorphic to the diffusion geometry  of the associated data points i.e.,
\begin{equation}
\label{eq:DM_G}
\timel{\map D}{t}(x_i,x_j) = \norm{\cols{\timel{G}{t}}{i}-\cols{\timel{G}{t}}{j}} ,\quad i,j\in[n].
\end{equation}
Although embedding of the dataset $\set X$ in $\Rn{n}$ by the columns of $\timel{G}{t}$ preserves the diffusion geometry,  it may be ineffective for large $n$  as was explained in Section~\ref{sec:intro}. Therefore, a dimensionality reduction is required.

Since the transition probabilities matrix $P$ (Eq.~\ref{eq:P}) is conjugated to the symmetric matrix
$
M\triangleq D^{-1/2}KD^{-1/2}
$
via the relation $P = D^{-1/2}MD^{1/2}$, $P$ has a complete real eigen-system. Moreover, due to Gershgorin's circle theorem~\cite{golub} and the fact that $P$ is stochastic, all its eigenvalues are lying in the interval $(-1,1]$ (the exclusion of $-1$ is due to the assumption that the chain is aperiodic). Let
\begin{equation}
\label{eq:Msvd}
M=USU^\ast
\end{equation}
be the eigen-decomposition of $M$, where $U$ is an orthogonal $n\times n$ matrix and $S$ is a diagonal $n\times n$  matrix, whose diagonal elements $s_i$ are ordered decreasingly due to their modulus $1=s_1>\abs{s_2}\geq\ldots\geq \abs{s_{n}}\geq 0$. The first inequality is due to the assumption that the graph is connected. Thus, following Eq.~\ref{eq:G}
\begin{equation}
\label{eq:Gt}
\timel{G}{t} = (M^\ast)^t D^{-1/2} = US^tU^\ast D^{-1/2}.
\end{equation}
Therefore, due to the orthogonality of $U$ and according to Eq.~\ref{eq:DM_G} $$\timel{\map D}{t}(x_i,x_j) = \norm{\vect d}_{1}^{1/2}\norm{S^tU^\ast D^{-1/2}(\ivect {e}{i}-\ivect {e}{j})},\quad i,j,\in[n],$$
where $\ivect{e}{i}$ denotes the $i$-th standard unit vector in $\Rn{n}$. As a consequence, the diffusion maps $\timel{\map\Psi}{t}:\set X\to\Rn{n}$ are defined as follows:
\begin{eqnarray}\label{eq:svd_psi}
\timel{\map \Psi}{t}(x_i) &\triangleq & \norm{\vect d}_1^{1/2}U^\ast \timel{G}{t}\ivect{e}{i}\\
& = & \norm{\mathbf{d}}_{1}^{1/2} S^tU^\ast D^{-1/2}\ivect{e}{i}\nonumber \\
& = & \norm{\mathbf{d}}_{1}^{1/2}\ventry{d}{i}^{-1/2}S^t(\rows{U}{i})^\ast\nonumber \\
& = & \ventry{\hat d}{i}^{-1/2}[s_1^t \mentry{U}{i}{1},\ldots,s_n^t \mentry{U}{i}{n}]^\ast\nonumber .
\end{eqnarray}
Of course, the diffusion maps provide the required embedding of $\set X$ in $\Rn{n}$, since their Euclidean geometry in $\Rn{n}$ is identical to the diffusion geometry of the dataset $\set X$, i.e. $\timel{\map D}{t}(x_i,x_j) =\norm{\timel{\map\Psi}{t}(x_i)-\timel{\map\Psi}{t}(x_j)},~i,j\in[n].$ In order to achieve a low-dimensional embedding, the diffusion map from Eq.~\ref{eq:svd_psi} is projected onto its significant principal components according to the decay rate of the spectrum of $M^t$. Specifically, for a sufficiently small $\abs{s_{k+1}}^t$, the $k$-dimensional embedding is $\map{T}_k\circ\map\Psi_t$, where $\map{T}_k:\Rn{n}\to\Rn{k}$ is the projection on the first $k$ coordinates. Lemma~\ref{lem:svd_distort} quantifies the distortion resulted by such a projection.
\begin{lemma}\label{lem:svd_distort}
Let $c=\max_{i\in[n]}\ventry{\hat d}{i}^{-1/2}$ (see Eq.~\ref{eq:d_hat}). Then the $k$-dimensional embedding $\map{T}_k\circ\timel{\map\Psi}{t}$ is a $\sqrt{2}c\abs{s_{k+1}}^{t}$-distortion of the $n$-dimensional diffusion map $\timel{\map\Psi}{t}$ from Eq.~\ref{eq:svd_psi}.
\end{lemma}

\begin{proof}
Following Eq.~\ref{eq:svd_psi} we get\footnote{Here, for comparison purposes, we use the convention that $\map T_k:\Rn{n}\to\Rn{n}$, is the operator that zeros out the last $n-k$ coordinates.}
\begin{eqnarray*}
\abs{\norm{\timel{\map\Psi}{t}(x_i)-\timel{\map\Psi}{t}(x_j)} - \norm{\map{T}_k\circ\timel{\map\Psi}{t}(x_i)-\map{T}_k\circ\timel{\map\Psi}{t}(x_j)}} &\leq &\norm{\timel{\map\Psi}{t}(x_i) - \map T_k\circ\timel{\map\Psi}{t}(x_i)}\\
& + &\norm{\timel{\map\Psi}{t}(x_j) - \map T_k\circ\timel{\map\Psi}{t}(x_j)}\\
& \leq & \abs{\ventry{s}{k+1}}^t(\ventry{\hat d}{i}^{-1}+\ventry{\hat d}{j}^{-1})^{1/2}\\
&\leq &\sqrt{2}c\abs{\ventry{s}{k+1}}^t.
\end{eqnarray*}
\end{proof}

The distortion bound from Lemma~\ref{lem:svd_distort} is referred to as the \emph{analytic bound}. In many cases, spectral properties of the utilized kernel are known a-priori with no need for its explicit computation. The Gaussian kernel is just one example (see~\cite{Bermanis201315, Bermanis2014302}) but not the only. In such cases, only a partial SVD can be calculated to produce the relevant principal components according to the required distortion.

\subsection{Efficient ICPQR-based DM framework for data analysis}
\label{subsec:QRBasedDM}

In this section, we provide a QR-based framework for low-dimensional representation of the DM for a training set $\set X$, its out-of-sample extension and anomaly detection. For this purpose, $\set X$ is assumed to be a subset of $\bar{\set X}$, on which a symmetric point-wise positive kernel $k: \bar{\set X}\times \bar{\set X}\to\Rn{}$ is defined.

\subsubsection{QR-based low-dimensional embedding}\label{sec:DMembed}
Equation~\ref{eq:DM_G} suggests that the diffusion geometry is already embodied in the Euclidean geometry of $\timel{G}{t}$'s columns $\timel{\set G}{t}\triangleq\{\cols{(\timel{G}{t})}{1},\ldots,\cols{(\timel{G}{t})}{n}\}$ (see Eq.~\ref{eq:G}).  According to Proposition~\ref{prop:main} and Eq.~\ref{eq:DM_G}, application of Algorithm~\ref{alg: q-less} to $\timel{G}{t}$ with distortion rate $\mu>0$ produces an $s$-dimensional $2\mu$-distortion $\map F_s: \timel{\set G}{t}\to\Rn{s}$, for which $$\max_{i,j\in[n]}\abs{\timel{\map D}{t}(x_i,x_j)-\norm{\timel{\map h}{t}_s(x_i)- \timel{\map h}{t}_s(x_j)} }\leq 2\mu,$$ where $\timel{\map h}{t}_s:\set X\to\Rn{s}$ is defined by
\begin{equation}
\label{eq:h_s}
\timel{\map h}{t}_s(x_i) \triangleq \map F_s(\cols{(\timel{G}{t})}{i}),\quad i\in[n].
\end{equation}
As was discussed in Section~\ref{sec:PIQR}, the embedding dimension $s$ is not known a-priori and is a non-increasing function of $\mu$.  Algorithm~\ref{alg: qrdm} summarizes the above.
\IncMargin{1em}
\begin{algorithm}
    \DontPrintSemicolon
    \SetAlgoLined
    \SetKwComment{tcp}{//}{}
    \SetKwInOut{Input}{Input}\SetKwInOut{Output}{Output}
    \Input{An $n\times n$ kernel matrix $K$, time step $t\in\mathbbm{N}$ and a nonnegative distortion parameter $\mu$.}
\Output{An $s$-dimensional $2\mu$-distortion $\timel{\map h}{t}_s:\set X\to \Rn{s}$ of $\timel{\map  \Psi}{t}$, a dictionary $\set D\subset\timel{\set G}{t}$ of $s$ elements, an $s\times n$ matrix $\bar{R}$ and a degrees vector $\vect d\in\Rn{n}$.}
    \BlankLine
    set $\vect d\in\Rn{n}$, $\ventry{d}{i} = \sum_{j=1}^n \mentry{K}{i}{j}$\\
    set the $n\times n$ diagonal matrix $D$, whose $i$-th diagonal element is $\ventry{d}{i}$ \\
    set the $n\times n$ row stochastic transition probabilities matrix in time $t$, $P^t=(D^{-1}K)^t$\\
    set $\timel{G}{t} = \norm{\vect d}_{1}^{1/2}D^{-1/2}(P^t)^\ast$ (see Eq.~\ref{eq:G})\\
    apply Algorithm~\ref{alg: q-less} to $\timel{G}{t}$ and $\mu$ to get a permutation $\mapi:[n]\to[n]$ and an $s\times n$ matrix $\bar R$\\
    define $\set D = \{\cols{(\timel{G}{t})}{\mapi(1)},\ldots,\cols{(\timel{G}{t})}{\mapi(s)}\}$ and $\timel{\map h}{t}_s(x_i) = \cols{\bar R}{i},~i\in[n]$. \label{idm: last_step}
    \caption{ICPQR-based DM}
    \label{alg: qrdm}
\end{algorithm}
\DecMargin{1em}

The output parameters $\bar R$ and $\vect d$ of Algorithm~\ref{alg: qrdm} are needed for the out-of-sample phase, described in Section~\ref{subsubsec: dm_oose}. In addition, due to Definition~\ref{def: normality2}, the anomaly detection in the DM context $\mu_{strict}$ takes the form
\begin{equation}
\label{eq:dm_mu_strict}
\timel{\mu}{t}_{strict} = \sup_{i\in[n]}\norm{\cols{(\timel{G}{t})}{i}-\timel{\map h}{t}_s(x_i)}.
\end{equation}

\subsubsection{Out-of-sample extension and anomaly detection}
\label{subsubsec: dm_oose}
Given an out-of-sample data point $x\in\bar{\set X} \backslash\set X$, the goal of the present section is to extend $\timel{\map h}{t}_s$ from Eq.~\ref{eq:h_s} to $x$.  For this purpose, a user-defined probabilities vector $\timel{\map p}{t}(x)\in\Rn{n}$ has to be defined. The $i$-th entry of $\timel{\map p}{t}(x)$ defines the transition probabilities from $x$ to $x_i\in\set X$ in $t$ time steps, i.e. $ {\timel{\map p}{t}(x)}_{i} = \mathbbm P(x(t)= x_i | x(0) = x),~i\in[n]$. Consequently, consistently with Eq.~\ref{eq:G}, the $n$ dimensional extension of $\timel{\set G}{t}$ to $x$ is defined by $\timel{\vect g}{t}(x) \triangleq \norm{\vect d}_{1}^{1/2}D^{-1/2}\timel{\map p}{t}(x)\in\Rn{n}$. Then, Algorithm~\ref{alg: q-less_oos} is applied to $\timel{\vect  g}{t}(x)$ to produce an $s$-dimensional embedding $\timel{\map h}{t}_s(x)=\map F_s(\timel{\vect  g}{t}(x))$.

This scheme is consistent with the low-dimensional embedding scheme, described in Section~\ref{sec:DMembed}, in the sense that if the probabilities vector $\timel{\map p}{t}(x)$ equals to an in-sample probabilities vector $\cols{(P^\ast)}{i}$ for a certain $i\in[n]$, then $\timel{\map h}{t}_s(x) = \timel{\map h}{t}_s(x_i)$. Algorithm~\ref{alg: dm_oose} summarizes the above.

\IncMargin{1em}
\begin{algorithm}
    \DontPrintSemicolon
    \SetAlgoLined
    \SetKwComment{tcp}{//}{}
    \SetKwInOut{Input}{Input}\SetKwInOut{Output}{Output}
    \Input{A dictionary $\set D$, an $s\times n$ matrix $\bar R$ and a degrees vector $\vect d\in\Rn{n}$,  which are the outputs from Algorithm~\ref{alg: qrdm} and a transition probabilities vector $\timel{\map p}{t}(x)\in\Rn{n}$.}
\Output{The extension of $\timel{\map h}{t}_s$ to $x$, ${\map h}_s^{(t)}(x)$ and the associated distortion rate $\bmu(x)$.}
    \BlankLine
    set $\timel{\map g}{t}(x)= \norm{\vect d}_{1}^{1/2}D^{-1/2}\timel{\map p}{t}(x)$, where $D$ is the diagonal $n\times n$ matrix $diag(\vect d)$\\
    apply Algorithm~\ref{alg: q-less_oos} to $\set D$, $\bar R$ and $\timel{\map g}{t}(x)$ to get ${\map h}_s^{(t)}(x)\in\Rn{s}$\\
    define $\bmu(x) = (\norminline{\timel{\map g}{t}(x)} ^2-\norminline{\timel{\map h}{t}_s(x)} ^2)^{1/2}$\label{stp:dist_rate}\\
    \caption{Out-of-sample extension for ICPQR-based DM}
    \label{alg: dm_oose}
\end{algorithm}
\DecMargin{1em}

One possibility for the definition of $\timel{\map p}{1}(x)$, which is the first time step transfer probabilities from an out-of-sample data point to $\set X$, is via the kernel function $k$ by
\begin{equation}\label{eq:probs_vect}
\timel{\map p}{1}(x)_{(i)} \triangleq k(x,x_i)/\sum_{j=1}^n k(x,x_j),~i\in[n].
\end{equation}
This definition is consistent with Eq.~\ref{eq:P}. Then, the corresponding transition probabilities vector in time step $t$ can be heuristically defined by $\timel{\map p}{t}=(P^\ast)^{t-1}\timel{\map p}{1}(x)$. This definition represents a Markovian process for which the new data point $x$ is inaccessible from the dataset $\set X$, and the transition probabilities from $x$ to $\set X$ in time-step $t$ are determined by the transition probabilities from $x$ to $\set X$ in the first time-step that are represented by $\timel{\map p}{1}(x)$, and the transition probabilities among the elements of $\set X$ after $t-1$ time-steps that are represented by $(P^\ast)^{t-1}$.  Finally, the (strict) normality of an out-of-sample data point $x$ is determined due to Definition~\ref{def: normality2} and Eq.~\ref{eq:dm_mu_strict} by using the distortion rate function $\bmu(x)$ from step~\ref{stp:dist_rate} in Algorithm~\ref{alg: dm_oose}.

\section{Experimental Results}
\label{sec:experimental}
This section analyses three different datasets using the proposed methodologies from Sections~\ref{sec:qrdr} and~\ref{sec:qrdm} that are synthetic and real. Section~\ref{subsec: swiss_roll} exemplifies the basic notions of geometry preservation, anomaly detection and out-of-sample extension through the application of the QR-based DM to a synthetic dataset as described in Section~\ref{sec:qrdm}. A comparison with the method proposed in~\cite{Salhov:uIDM2014} for diffusion geometry preservation is presented in this section as well. A QR-based DM analysis of real data is demonstrated in Section~\ref{subsec:exp_darpa}. The analysis in both of the above examples is based on the corresponding first time step in DM. Finally, Section~\ref{subsec:isolet} presents a multiclass classification of parametric data, using generalizations of the out-of-sample extension and anomaly detection methods, presented in Section~\ref{sec: oose}.

\subsection{QR-based DM analysis - toy example} \label{subsec: swiss_roll}

In this section, we present a diffusion-based analysis of a synthetic two dimensional manifold, immersed in a three dimensional Euclidean space. The analyzed dataset $\set X\subset\Rn{3}$ consists of $n=3,000$ data points, uniformly sampled  from a Swiss roll, shown in Fig.~\ref{fig:swiss_roll}.
\begin{figure}[H]
    \centering
    \includegraphics[width=0.4\textwidth,keepaspectratio]{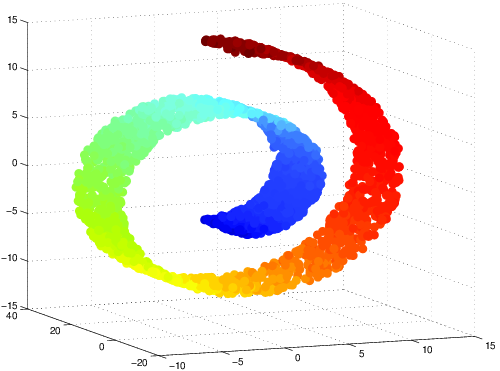}
    \caption{Swiss roll that contains $3,000$ uniformly distributed data points. Data points are colored according to their distance from the origin.}
    \label{fig:swiss_roll}
\end{figure}
The utilized kernel function is the commonly-used Gaussian kernel $k_\ve:\set X\times \set X\to\Rn{}$,
\begin{equation}
\label{eq:Gkernel}
k_\ve(x,y)\triangleq e^{-\norm{x-y} ^2/\ve},\quad \ve>0,
\end{equation}
where the norm in the exponent is the standard three dimensional Euclidean norm. The associated $n\times n$ kernel matrix is $K_\ve$ whose $(i,j)$-th entry is $$
\mentry{(K_\ve)}{i}{j}\triangleq k_\ve(x_i,x_j),\quad i,j\in[n].
$$
The corresponding degrees function is $\vect d_\ve\in\Rn{n}$, whose $i$-th entry is $\ventry{(d_\ve)}{i}\triangleq\sum_{j=1}^n k_\ve(x_i,x_j), ~i\in[n]$ and $D_\ve$ is the $n\times n$ diagonal matrix, whose $i$-th diagonal entry is $\ventry{(d_\ve)}{i}$. Based on these, according to Eq.~\ref{eq:P}, the $n\times n$ transition probabilities matrix is defined by
\begin{equation}
\label{eq:P_ve}
P_\ve\triangleq D_\ve^{-1}K_\ve.
\end{equation}
Thus, transition probabilities between close data points are high and low for far points.

Section~\ref{sec:vevss} addresses the qualitative dependency between the neighborhood parameter $\ve$, and the required embedding's dimensionality. Section~\ref{sssec: lde} shows a further step of dimensionality reduction, using the optimal coordinates system, as described in Section~\ref{sec:PIQR}. Out-of-sample extension and anomaly detection, as described in Section~\ref{subsubsec: dm_oose}, are demonstrated in Section~\ref{subsec:expr_oose_anomaly}. Finally, Section~\ref{subsec:cmp_idm} presents a brief description of the $\mu$-IDM method~\cite{Salhov:uIDM2014} and compares its performances to the proposed  method in this paper.

\subsubsection{The dependency between $\ve$ and the embedding's dimension $s$}\label{sec:vevss}
As was proved in~\cite{Bermanis201315}, as $\ve$ increases, the numerical rank of $P_\ve$ decreases and vice versa. Mathematically, let $1=s^{(\ve)}_1\geq s^{(\ve)}_2\geq\ldots\geq s^{(\ve)}_n\geq 0$ be the eigenvalues\footnote{The eigenvalues of $P_\ve$ are nonnegative since the Gaussian kernel function $k_\ve$ from Eq.~\ref{eq:Gkernel} is positive definite due to Bochner's theorem~\cite{scatter}. Thus, if the data points in $\set X$ are all distinct, then $k_\ve$ is strictly positive definite and the eigenvalues of $P_\ve$ are all positive.} of $P_\ve$. Define $\map E_\ve(r):[0,1]\to[0,1]$, $\map E_\ve(r) \triangleq (\sum_{i=1}^{t}{(s^{(\ve)}_{i})^{2}}/\sum_{i=1}^{n}{(s^{(\ve)}_{i})^{2}})^{1/2}$ to be the energy's portion of $P_\varepsilon$, which is captured by the first $t$ eigenvalues of $P_\ve$, where $r=t/n$ is the corresponding spectrum ratio. Then, as $\ve$ increases, the number of significant eigen-components decreases as demonstrated in Fig.~\ref{fig:dm_decay}. This fact, combined with Lemma~\ref{lem:svd_distort}, suggests that when $\ve$ decreases, the number of components required to achieve a certain distortion increases as seen in Fig.~\ref{fig:sw_embedding_dim}.
\begin{figure}[H]
\centering
\subfigure[Spectra of $P_\varepsilon$]{
\includegraphics[width=0.3\textwidth,keepaspectratio]{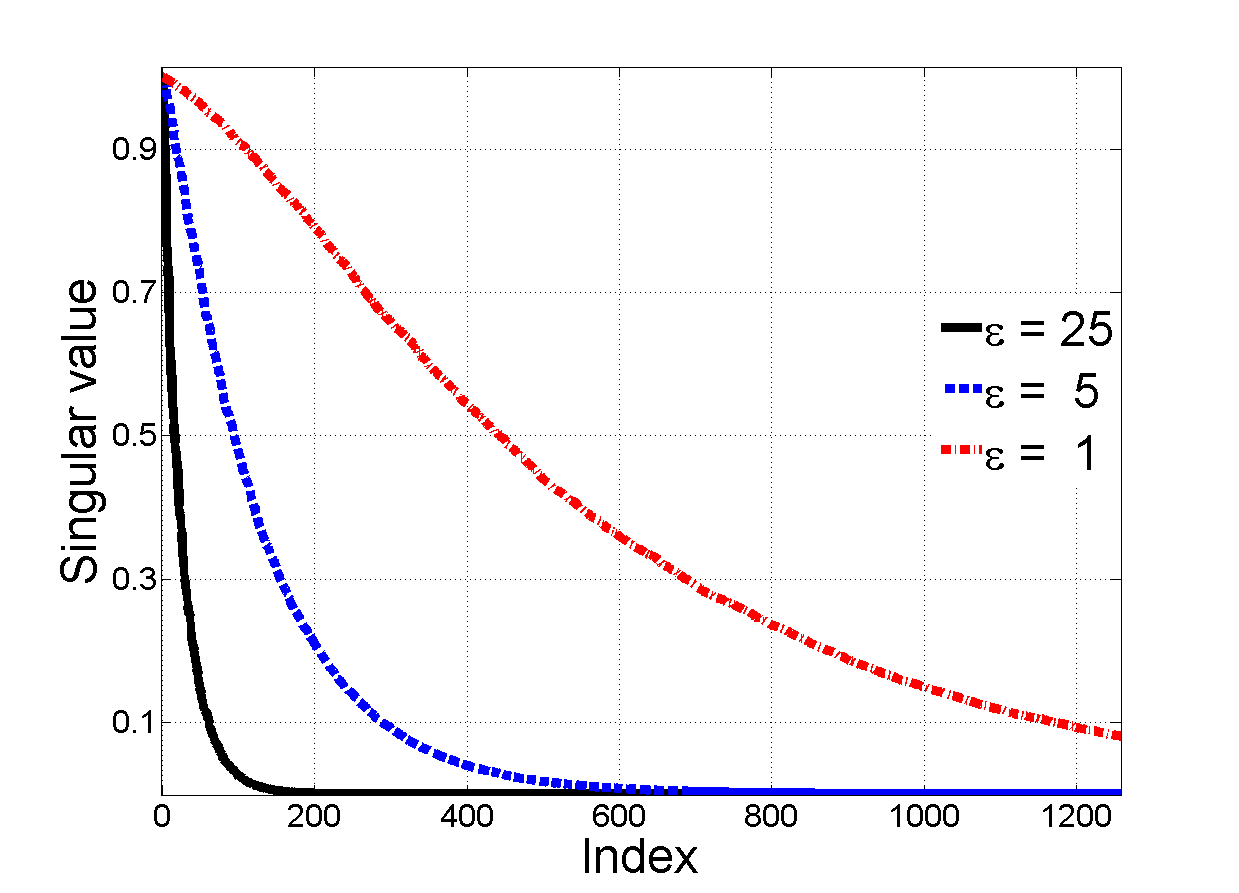}
\label{fig:specs}
}
\subfigure[Spectra ratios of $P_\varepsilon$]{
\includegraphics[width=0.3\textwidth,keepaspectratio]{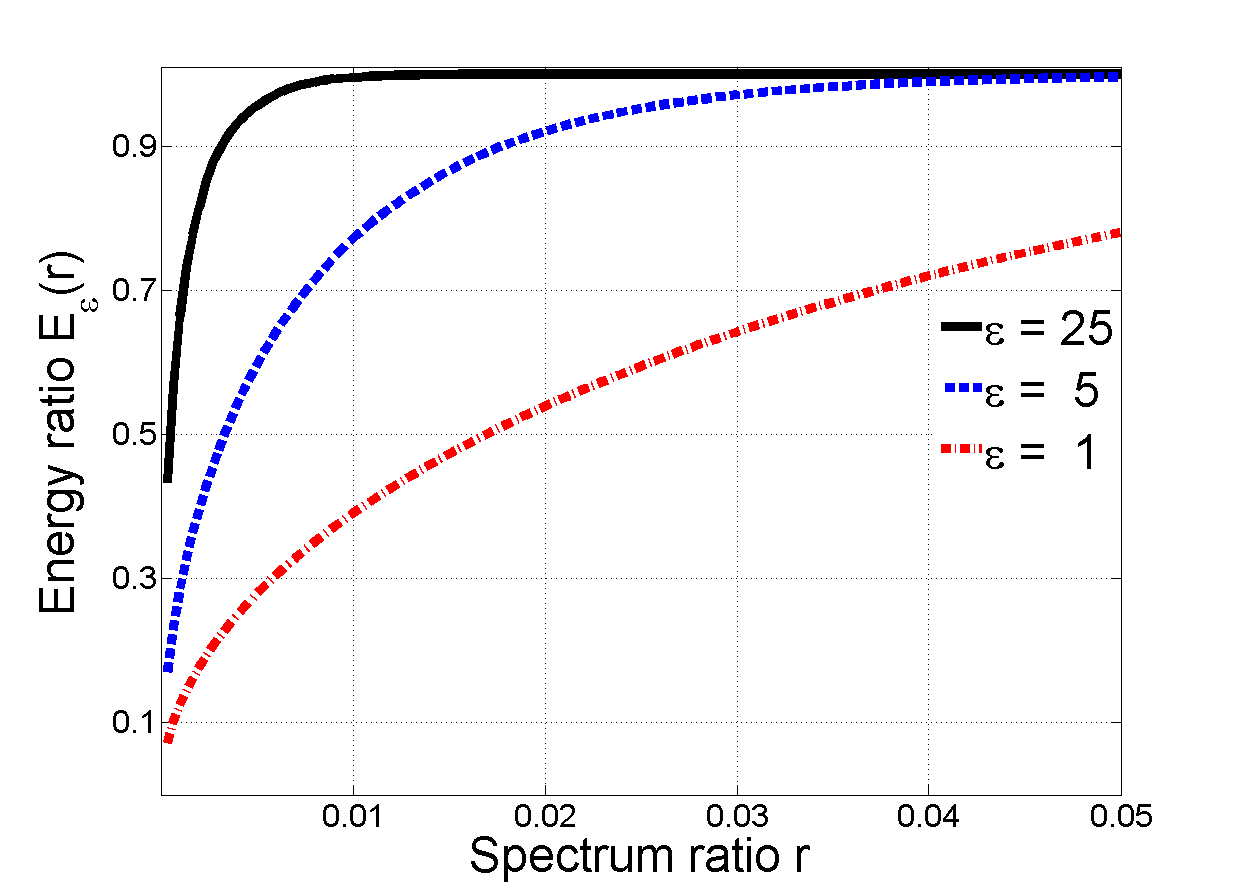}
\label{fig:energy}
}
\subfigure[Diffusion distances distributions]{
\includegraphics[width=0.3\textwidth,keepaspectratio]{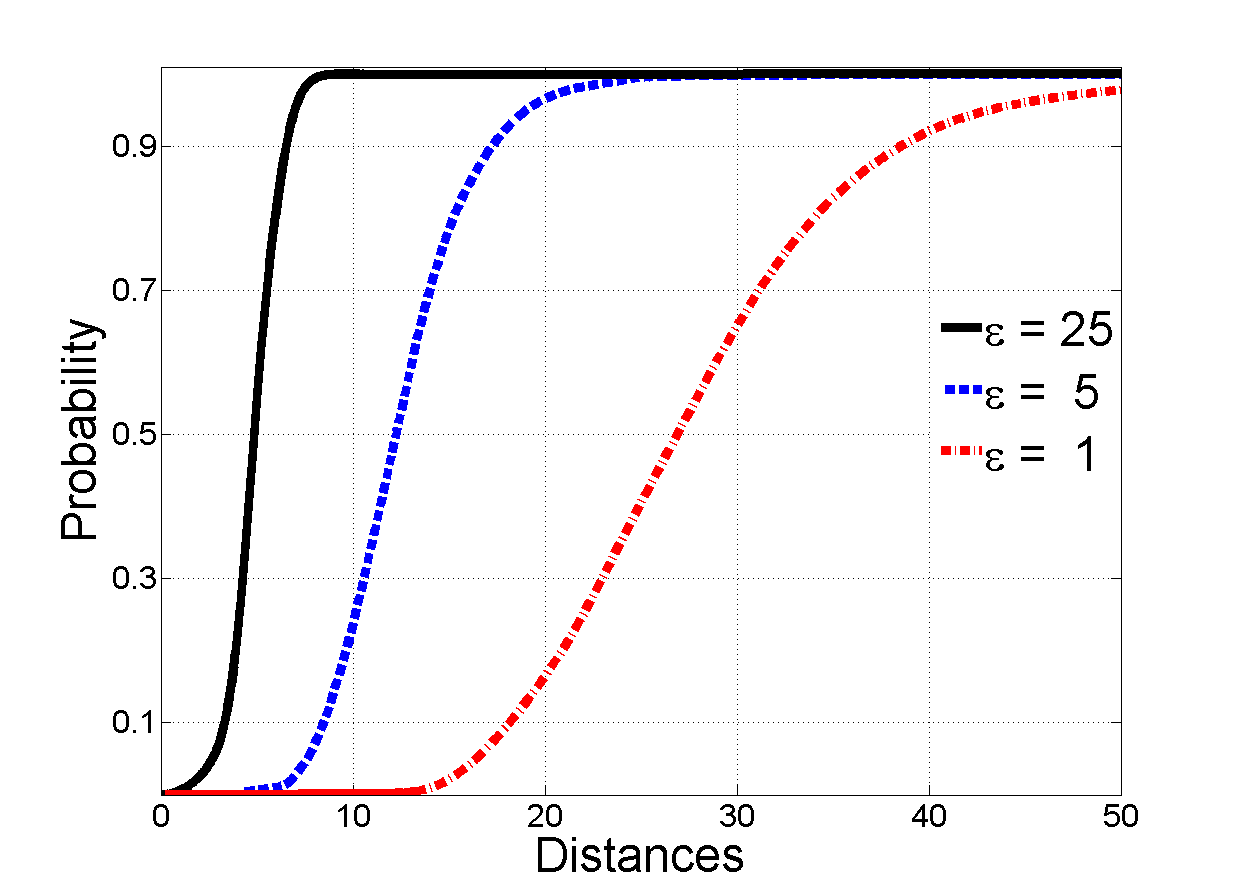}
\label{fig:hist} } \caption{Spectral and geometrical views of three diffusion geometries that correspond to the neighborhood parameters $\varepsilon=1,5$ and $25$. Figure~\subref{fig:specs} shows that as $\varepsilon$ becomes larger the spectrum decays faster. An immediate consequence is shown in Fig.~\subref{fig:energy} that shows the relation between the number of significant components and $\varepsilon$. Figure~\subref{fig:hist} shows the probability distribution of the diffusion distances. It is clear that the use of large $\epsilon$ results in many short diffusion distances and vice-versa.} \label{fig:dm_decay}
\end{figure}

Figure~\ref{fig:sw_embedding_dim} compares between the analytic bound (see Lemma~\ref{lem:svd_distort}), the minimal dimension of DM and the QR-based DM dimension that are required to achieve a certain distortion. It also demonstrates the above discussed relation between the neighborhood parameter $\ve$ and the dimensionality of the embedding. Thus, for a larger $\ve$ a fewer dimensions are required to achieve a certain distortion.
\begin{figure}[H]
\centering
\subfigure[Small neighborhood, $\varepsilon=1$]{
\includegraphics[scale = 0.15]{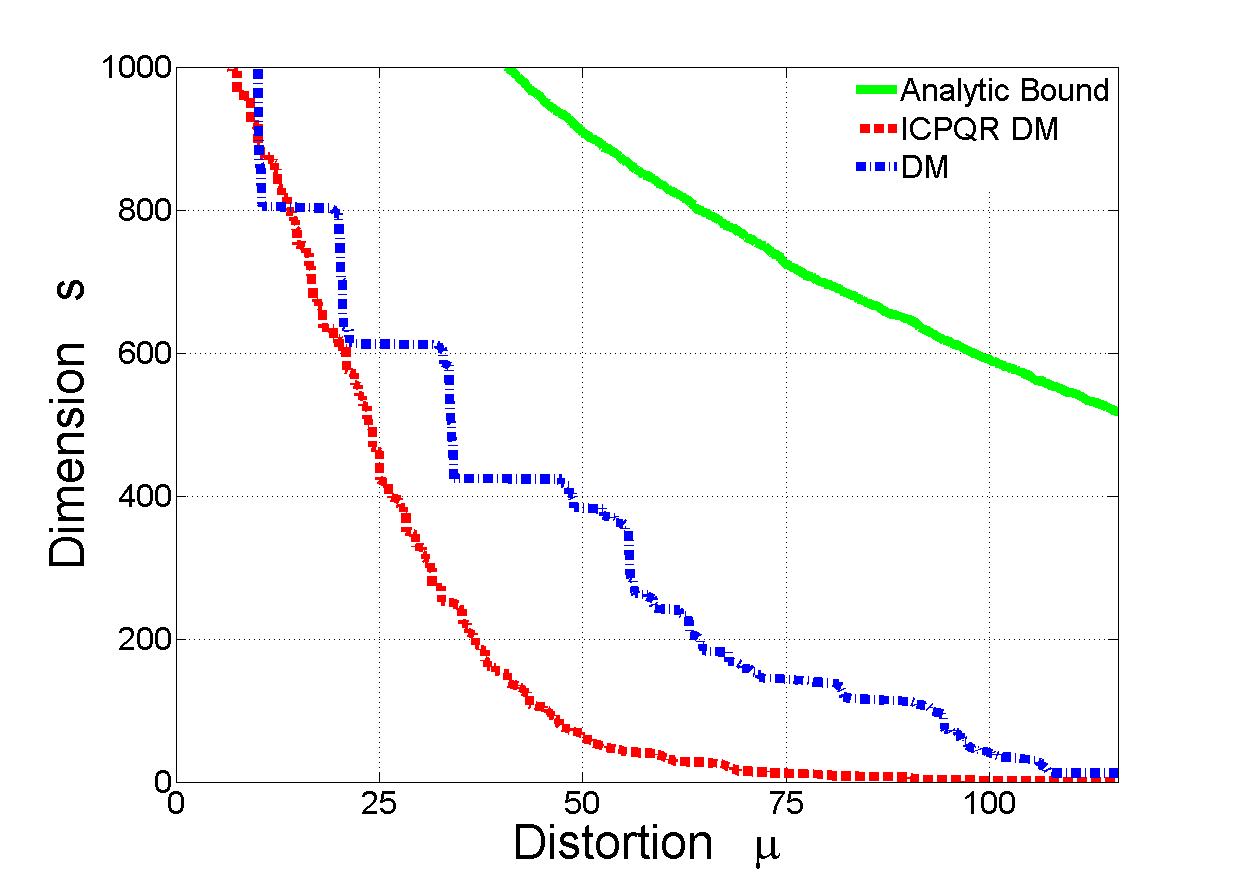}
\label{fig:sw_slow}}
\subfigure[Medium neighborhood, $\varepsilon=5$]{
\includegraphics[scale=.15]{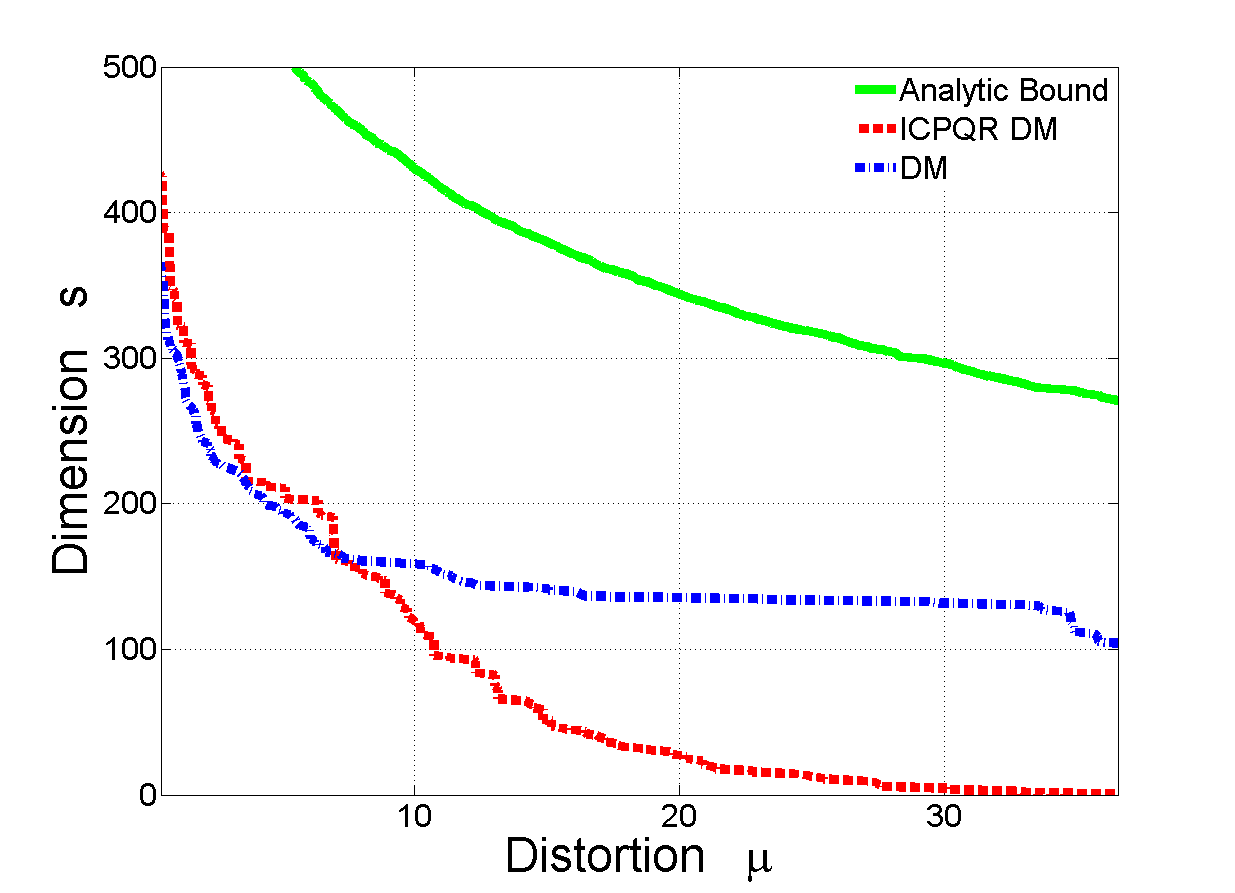}
\label{fig:sw_medium}}
\subfigure[Large neighborhood, $\varepsilon=25$]{
\includegraphics[scale=.15]{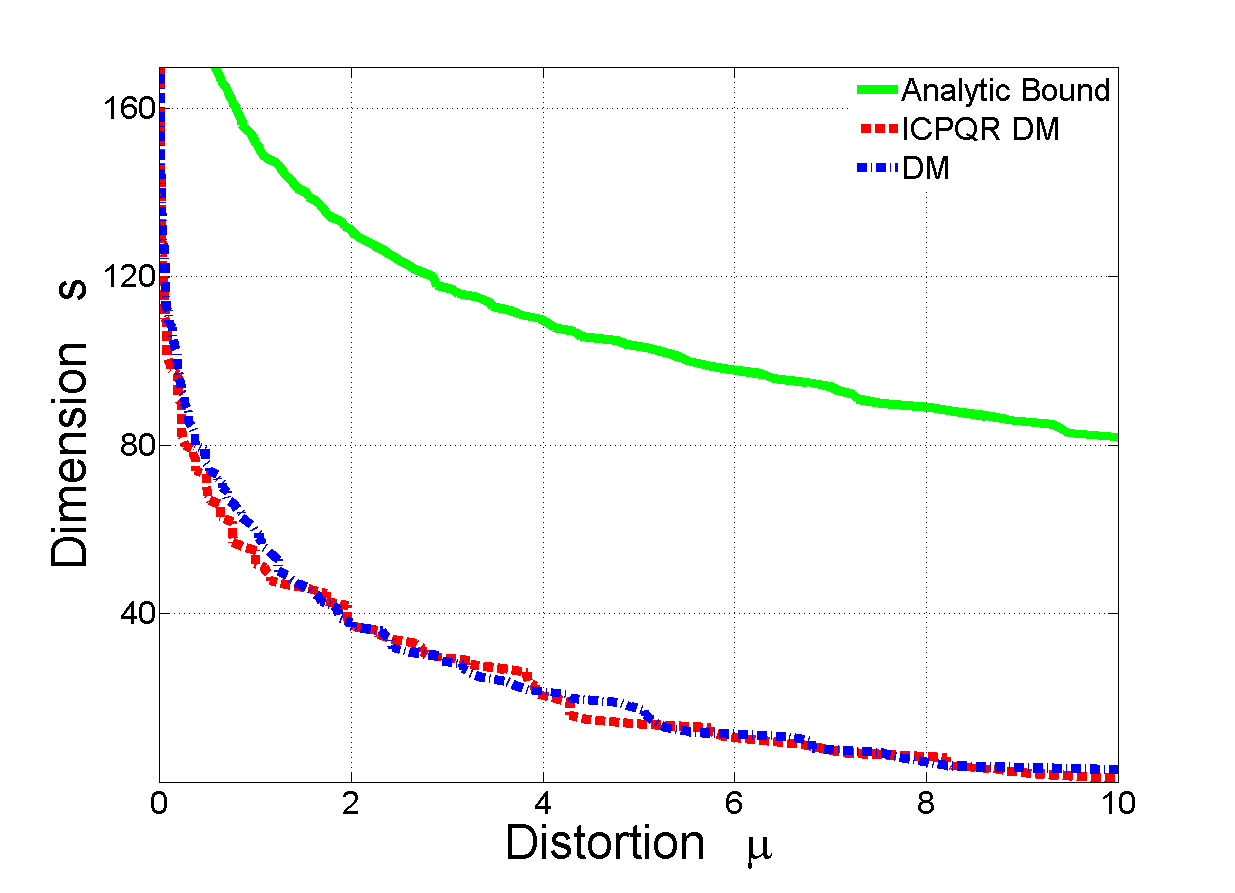}
\label{fig:sw_fast}}
\caption{Number of components (dimensions- $y$-axis) required to preserve the Swiss roll diffusion geometry up to a distortion ($x$-axis) for three different neighborhood sizes. The continuous (green) graph denotes the analytic bound provided by Lemma~\ref{lem:svd_distort}, the dashed (red) graph is the QR-based DM dimension produced by Algorithm~\ref{alg: q-less} and the dash-dotted (blue) graph is the minimal DM dimension required to achieve a certain distortion.} \label{fig:sw_embedding_dim}
\end{figure}

\subsubsection{Low-dimensional embedding}\label{sssec: lde}
In this section, a comparison between the classic DM and ICPQR-based DM is presented. Figure~\ref{fig:icpqr_swiss_roll} shows the two-dimensional DM embedding of $\set X$, and a two-dimensional view of an aligned versions of ICPQR-based DM, applied to $\set X$ with three different distortion values.

\begin{figure}[H]
\centering
\hspace{4pt}
\subfigure[DM]{
\includegraphics[scale=.21]{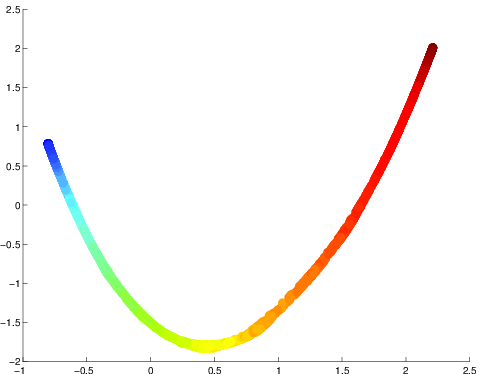}
\label{fig:dm_swiss_roll}}
\subfigure[ICPQR-DM, $\mu = 0.1$]{
\includegraphics[scale=.21]{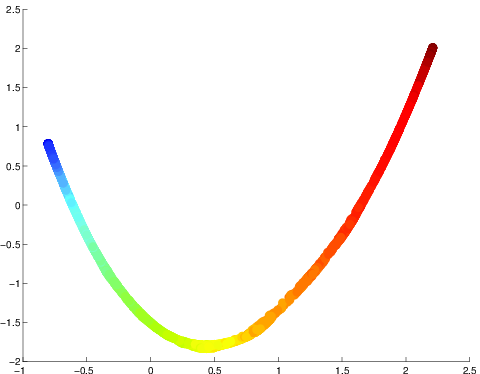}
\label{fig:icpqr_dm_swiss_roll_01}}
\subfigure[ICPQR-DM, $\mu = 1$]{
\includegraphics[scale=.21]{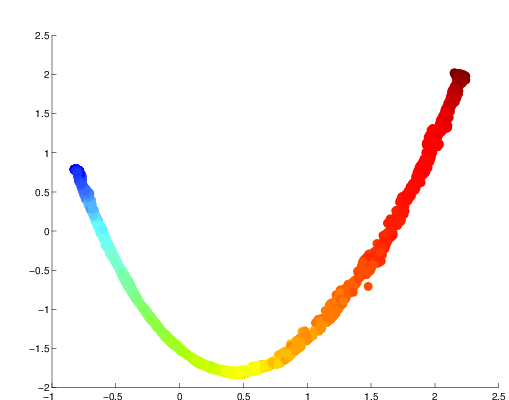}
\label{fig:icpqr_dm_swiss_roll_1}}
\subfigure[ICPQR-DM, $\mu = 5$]{
\includegraphics[scale=.21]{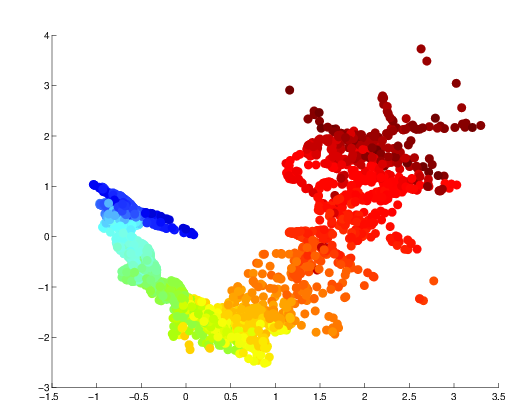}
\label{fig:icpqr_dm_swiss_roll_5}} \caption{Comparison between the two-dimensional DM and the projected aligned ICPQR-based DM of $\set X$ with $\ve=3$, for the first diffusion time-step:~\subref{fig:dm_swiss_roll}~two most significant DM coordinates of $\timel{\map\Psi}{1}$.~\subref{fig:icpqr_dm_swiss_roll_01}-\subref{fig:icpqr_dm_swiss_roll_5}: two most significant coordinates of the aligned version of $\timel{\map h}{1}_s$, with~\subref{fig:icpqr_dm_swiss_roll_01}~$\mu = 0.1$, $s=1,246$, actual distortion $0.01$,~\subref{fig:icpqr_dm_swiss_roll_1}~$\mu = 1$, $s=752$, actual distortion $0.23$, and~\subref{fig:icpqr_dm_swiss_roll_5}~$\mu = 5$, $s=382$, actual distortion $3.91$. Data coloring is consistent with Fig.~\ref{fig:swiss_roll}.} \label{fig:icpqr_swiss_roll}
\end{figure}

It should be stressed that $\timel{\map h}{1}_s$, which is the output of Algorithm~\ref{alg: qrdm}, is an $s$-dimensional $2\mu$-distortion of the $n$-dimensional DM $\timel{\map\Psi}{1}$ as defined in Definition~\ref{def: distortion}. Therefore, it is unlikely that low-dimensional projections (lower than $s$) will represent similar geometries. Nevertheless, to demonstrate the notion of low-rate distortion, Fig.~\ref{fig:icpqr_swiss_roll} shows a two dimensional view of an aligned version of $\timel{\map\Psi}{1}$ with $\timel{\map\Psi}{1}$. For that purpose, the DM of $\set X$ was explicitly computed. The utilized alignment algorithm is described in Appendix~\ref{appx}.

\subsubsection{Out of sample extension and anomaly detection}
\label{subsec:expr_oose_anomaly} Application of Algorithm~\ref{alg: qrdm} to $\set X$ with $K_\ve$, $\ve=3$, $t=1$ and $\mu=0.1$ was resulted in an $s$-dimensional $2\mu$-distortion of $\timel{\map\Psi}{1}$, $\timel{\map h}{1}_s:\set X\to\Rn{s}$ with $s=1,246$. An out-of-sample extension of $\timel{\map h}{1}_s$ to $\bar{\set X}$ is demonstrated in this section, as well as anomaly detection where $\bar{\set X}\subset \Rn{3}$ is a random subset of $10,000$ data points that are uniformly sampled from the three dimensional bounding box of $\set X$.

For that purpose, Algorithm~\ref{alg: dm_oose} was applied to $\bar{\set X}$. Beside its three first inputs, which are provided as outputs from Algorithm~\ref{alg: qrdm}, a transition probabilities vector $\timel{\map p}{1}(x)\in\Rn{n}$ has to be defined for any $x\in\bar{\set X}$. In this example, $\timel{\map p}{1}(x)$ was defined consistently with the kernel $K_\ve$ by using Eq.~\ref{eq:probs_vect}. This definition coincides with the definition of the transition probabilities matrix $P_\ve$ in Eq.~\ref{eq:P_ve} that results in exact extension on $\set X$, i.e. if $x=x_i$ for a certain $i\in[n]$, then  ${\map h}_s^{(1)}(x) = \timel{\map h}{1}_s(x_i)$.

The results are shown in Fig.~\ref{fig:oose_anomaly}. Figure~\ref{fig:sw_anomaly_measure} shows a side view of the dataset $\bar{\set X}$. Each data point $x\in\bar{\set X}$ is colored proportionally to its out-of-sample extension distortion rate $\bmu(x)$ (see step~\ref{stp:dist_rate} in Algorithm~\ref{alg: dm_oose}.) Classification of $\bar{\set X}$ to either normal or abnormal classes is shown in Fig.~\ref{fig:sw_normal_oos}. The normal class $\set N\subset$ is darkly colored and the abnormal class $\bar{\set N}$ is brightly colored. The classification was done according to Definition~\ref{def: normality2}. Thus, $x\in\bar{\set X}$ is classified as normal if its distortion rate satisfies $\bmu(x)\leq\mu$. Otherwise, it is classified as  abnormal. Lastly, a two dimensional view of the out-of-sample extension of the normal class, namely ${\map h}_s^{(1)}(x),~x\in{\set N}$, is shown in Fig.~\ref{fig:sw_concentrated_nndata}. Each embedded data point is colored in the same color as its nearest neighbor from the embedding of $\set X$ by ${\map h}_s^{(1)}({\set X})$. The shown coordinates system is consistent with the one presented in Fig.~\ref{fig:icpqr_swiss_roll}.
\begin{figure}[H]
\centering
\subfigure[Distortion rate function \break $\bmu:\bar{\set X}\to\Rn{}$.]{
\includegraphics[scale=.24]{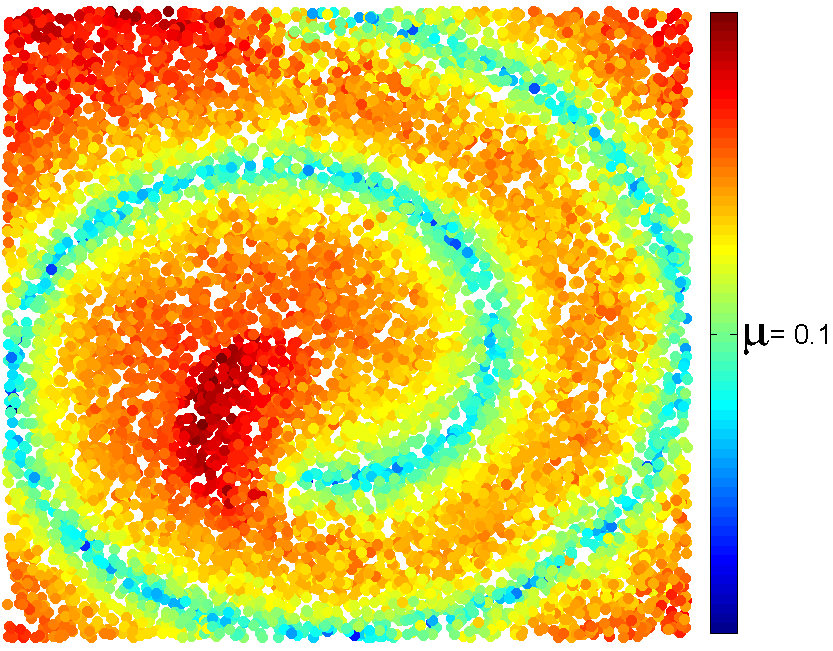}
\label{fig:sw_anomaly_measure}}
\subfigure[Classification of $\bar{\set X}$ to \break normal ($\bmu(x)\leq 0.1$) and \break abnormal (otherwise) classes.]{
\includegraphics[scale=.24]{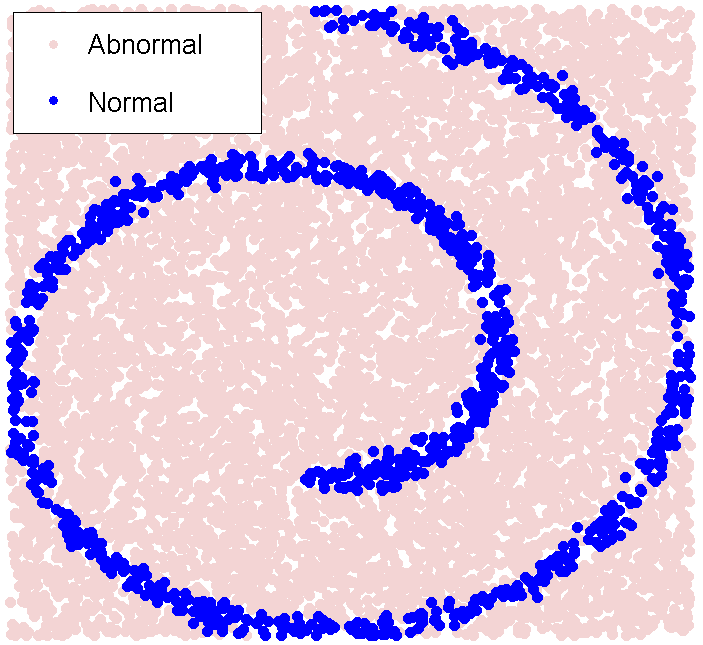}
\label{fig:sw_normal_oos}} \subfigure[Normal data points in the embedded space. Points are colored identically to their nearest neighbor in the original space (see Figure~\ref{fig:swiss_roll}) .]{ \includegraphics[scale=.20]{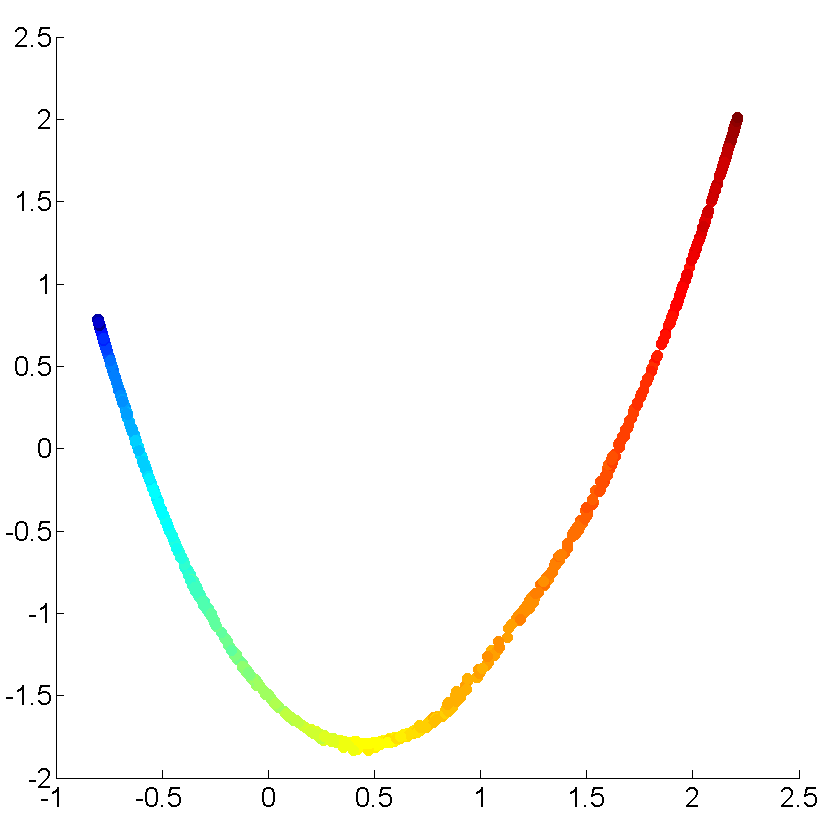}
\label{fig:sw_concentrated_nndata}}
\caption{$10,000$ out-of-sample data points were randomly selected in the bounding box of the original Swiss roll. Left and center: two-dimensional side view of the out-of-sample dataset colored by their distortion rate and classification, respectively. Right: extension of the first two meaningful ICPQR-based DM embedded space coordinates to the normal out-of-sample data points.} \label{fig:oose_anomaly}
\end{figure}

\subsubsection{Comparison with $\mu$IDM}
\label{subsec:cmp_idm}

The $\mu$IDM algorithm in~\cite{Salhov:uIDM2014} is a dictionary-based method that provides a low rank $2\mu$-distortion for the first transition time step $\timel{\map\Psi}{1}$ of DM. The algorithm incrementally constructs an approximated map by using a single scan of the data. This algorithm is greedy and sensitive to the scan order. Typically, the growth rate of the dictionary is very high at the beginning and decays as time advances. Moreover, the resulted dictionary and the resulted embedding's dimension may be redundant. In each iteration of the $\mu$IDM, a newly processed data point is considered for inclusion in the dictionary that was constructed from previously scanned data points.
At the beginning of every iteration, the dictionary elements are already embedded in a low-dimensional space (whose dimension equals to the size of the dictionary), where its geometry is identical to the diffusion geometry of the data restricted to the dictionary. Then, a Nystr\"{o}m-type extension~\cite{baker:nystrum} is applied to the scanned data points, based on its affinities with the dictionary elements, in order to approximate the embedding of the newly processed data point. The exact DM of this data point together with the dictionary is then efficiently computed. The geometries of these two embeddings are identical for the dictionary, therefore, at this stage these two geometries are aligned to coincide on the dictionary. The distance of the extended map from the exact map of the examined data point is measured. If this distance is larger than $\mu$ then the examined data point is added to the dictionary.  The entire computational complexity of this iterative process is lower than the computational complexity of DM. The exact number of required operations depends on the required accuracy and on the dimensionality of the original ambient space.

The presently proposed QR-based DM method considers the entire dataset in each iteration and it is not sensitive to the order of the dataset. Therefore, the resulted dictionary is more sparse as demonstrated in Table~\ref{tbl:compare_icpqr_idm} and in Fig.~\ref{fig:dict_sparse}.

\begin{table}[H]
\begin{center}
    \begin{tabular}{ |c|l|c|c| }
    \hline
        $\mu$    &  & ICPQR-based DM   & $\mu$IDM~\cite{Salhov:uIDM2014} \\
    \hline
        
    &  Dictionary size $s$            & $1,246$     & $2,305$      \\
$0.1$    &  Execution time          & $43$ sec.      & $7$ hours  \\
    &  Actual distortion           & $0.01$     & $0.03$   \\
    \hline
   
    &  Dictionary size $s$              & $752$      & $1,293$     \\
$1$     &  Execution time          & $27$ sec.       & $71$ minutes  \\
    &  Actual distortion           & $0.23$     & $0.61$   \\
    \hline
      
    &  Dictionary size  $s$             & $382$      & $630$       \\
$5$    &  Execution time          & $17$ sec.      & $15$ minutes  \\
    &  Actual distortion           & $3.91$     & $4.46$   \\
    \hline
       
    &  Dictionary size   $s$            & $190$      & $284$       \\
$10$    &  Execution time          & $9$ sec.       & $4$ minutes  \\
    &  Actual distortion           & $12.81$    & $13.24$  \\
    \hline
    \end{tabular}
\caption{Comparison between ICPQR-based DM and $\mu$IDM algorithms related to dictionary size, execution time and actual distortion\protect\footnotemark, w.r.t. $\timel{\map\Psi}{1}$, for various distortion parameters. Clearly, the actual distortion is bounded by $2\mu$. Execution times are averaged over $10$ runs of the algorithms.} \label{tbl:compare_icpqr_idm}
\end{center}
\end{table}
\footnotetext{An actual distortion of a function $f$ w.r.t. $g$ is $\sup_{x,y\in\set X}\absinline{\norminline{f(x)-f(y)}-\norminline{g(x)-g(y)}}$.}

\begin{figure}[H]
\centering \subfigure[ICPQR-based DM Dictionary, $382$ data points]{
\includegraphics[scale=.32]{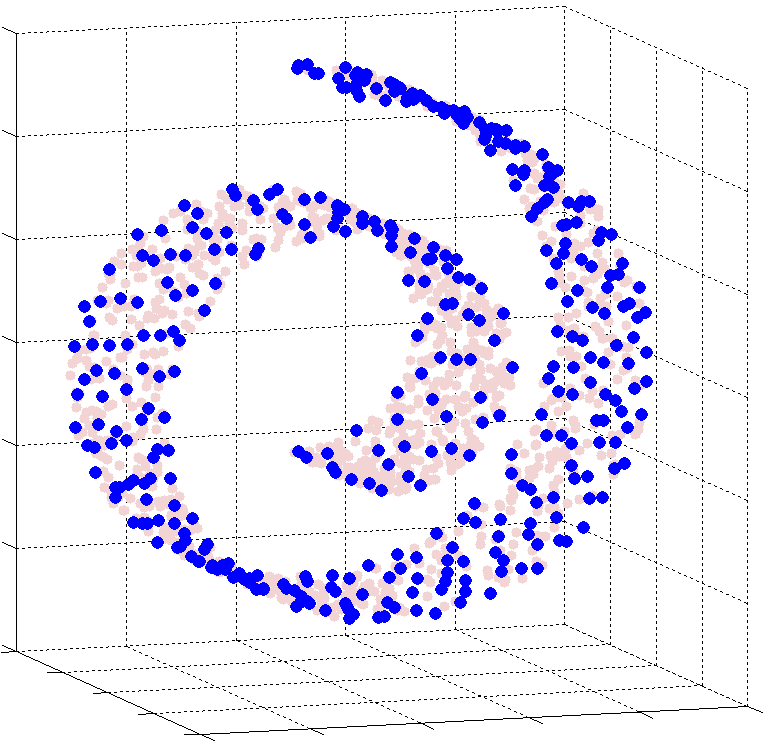}
\label{fig:dict_sparse_icpqr}} \subfigure[$\mu$IDM Dictionary, $630$ data points]{
\includegraphics[scale=.32]{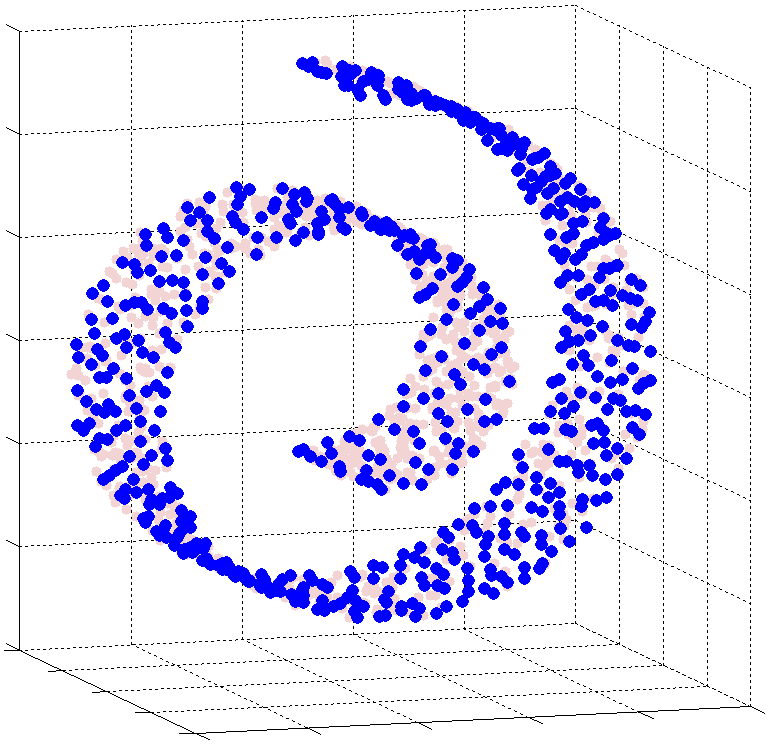}
\label{fig:dict_sparse_idm}} \caption{Data points admitted to ICPQR-based DM and $\mu$IDM dictionaries (dark points) with distortion parameter $\mu=5$. The input was given to both algorithms ordered from the inner part of the Swiss roll to the outer with a larger radius.} \label{fig:dict_sparse}
\end{figure}

\subsection{QR-based DM analysis - real-world data}
\label{subsec:exp_darpa}

This section exemplifies a semi-supervised anomaly detection process applied to a real-world dataset. The examined dataset $\bar{\set X}\subset\Rn{14}$ is the DARPA dataset~\cite{graf1998results} consists of $\bar n = 12,617$ data points. Each data point is a vector of $14$ features that describes a computer network traffic labeled as either normal and legitimate or abnormal that pertains to be an attack and intrusion on the network. The dataset is divided into a training set $\set X\subset \bar{\set X}$, which contains $n = 6,195$ normally behaving samples, and $5$ testing datasets $\{\set T_{mon},\ldots,\set T_{fri}\}$, which were collected during different days of the week. Each day contains normal and abnormal data points as described in Table~\ref{tbl:darpa_testsets}.

First, the training dataset was scaled to place it in the $14$-dimensional unit box $[0,1]^{14}$. Then, the same scaling was applied to the testing datasets $\set T=\set T_{mon}\cup\ldots\cup\set T_{fri}$. Application of Algorithm~\ref{alg: qrdm} to the training set $\set X$, using the Gaussian kernel from Eq.~\ref{eq:Gkernel} with $\ve=0.6$, $t=1$ and $\mu=10^{-7}$, produced an $s$-dimensional $2\mu$ distortion of the first time step $\timel{\map\Psi}{1}$ of DM  of $\set X$, $\timel{\map h}{1}_s:\set X\to\Rn{s}$ with $s=138$. The neighborhood size parameter $\varepsilon$ was chosen to be twice the median of all the mutual distances between the data points in $\Rn{14}$. Such a selection is a common heuristic for determining this  parameter in DM context. This concludes the training phase.

As a second stage,  for each testing data point $x\in\set T$, the probabilities vector $\timel{\vect p}{1}(x)$ was defined as a natural extension of the Gaussian kernel from Eq.~\ref{eq:Gkernel} by using Eq.~\ref{eq:probs_vect} with $\ve = 0.6$. Any testing data point, which is distant (relatively to $\ve$) from the training set $\set X$, yields $k_\ve(x,y)\approx 0,~y\in\set X$, and was a-priori classified as abnormal. The subset of distant testing data points is denoted by $\set F$.

Finally, Algorithm~\ref{alg: dm_oose} was applied to the elements of $\set T\backslash\set F$, using the previously computed probabilities vectors $\timel{\vect p}{1}(x),~x\in\set T$, to get an extension of $\timel{\map h}{1}_s:\set T\backslash\set F\to\Rn{s}$ and a distortion rate $\bmu:\set T\backslash\set F\to\Rn{}$. Then, abnormal data points were detected by following Definition~\ref{def: normality2} with $\mu_{strict}$ from Eq.~\ref{eq:dm_mu_strict}.

The anomaly detection results are summarized in Table~\ref{tbl:darpa_testsets}. Some of them are demonstrated in Fig.~\ref{fig:darpa_anomalies}. The rates in the accuracy percentage and in the false alarms columns are related to the original labeling of $\bar{\set X}$.
\begin{table}[ht]
    \centering
    \begin{tabular}{|c|c|c|c|c|}
        \hline Set & Size & \# of anomalies & Accuracy [\%] & False Alarms [\%] \\
        \hline $\set T_{mon}$ & $1,321$ & $1$ & $100$     & $0.68$\\
        \hline $\set T_{tue}$ & $1,140$ & $53$ &$100$    & $0.53$\\
        \hline $\set T_{wed}$ & $1,321$ & $16$ &$100$  & $0.08$\\
        \hline $\set T_{thu}$ & $1,320$ & $24$ &$96$    & $1.74$\\
        \hline $\set T_{fri}$ & $1,320$ & $18$ &$100$     & $0.15$\\
        \hline
    \end{tabular}
    \caption{Anomaly detection performances.}
\end{table}\label{tbl:darpa_testsets}

Figure~\ref{fig:darpa_anomalies} presents three-dimensional views of an aligned version of $\timel{\map h}{1}_s$ of $\set X$, as well as its extension to $\set T_{thu}\backslash\set F$ with the first three significant coordinates of $\timel{\map \Psi}{1}$, the DM of $\set X$ in the first time step. As can be seen in Fig~\ref{fig:darpa_anomalies}, most of data is located near the training set while the abnormal data points are embedded far away.

\begin{figure}[H]
\centering
\subfigure[]{
\includegraphics[scale=.15]{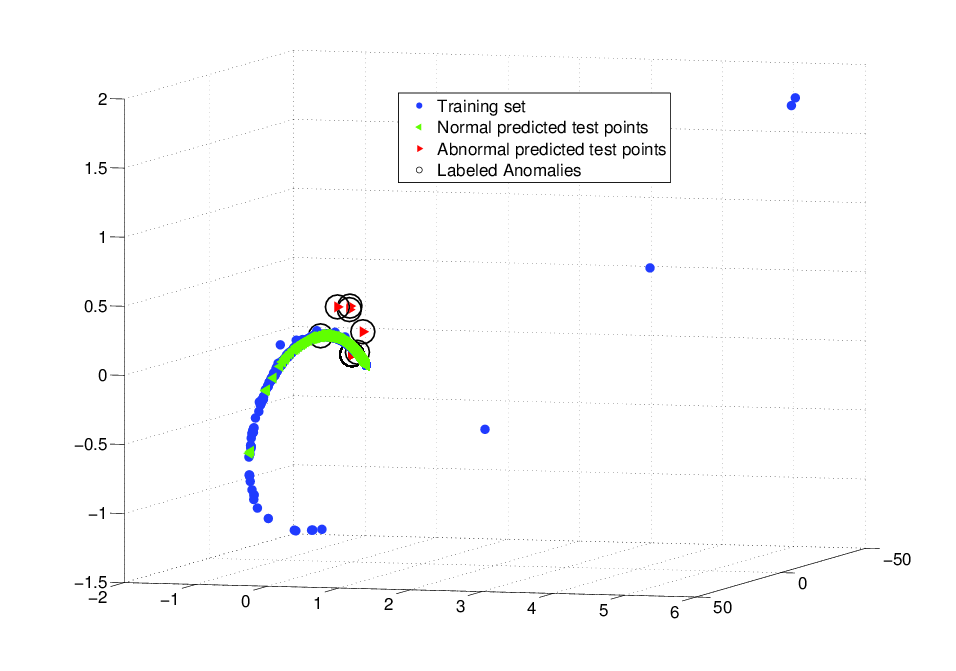}
\label{fig:darpa_general}}
\subfigure[]{
\includegraphics[scale=.15]{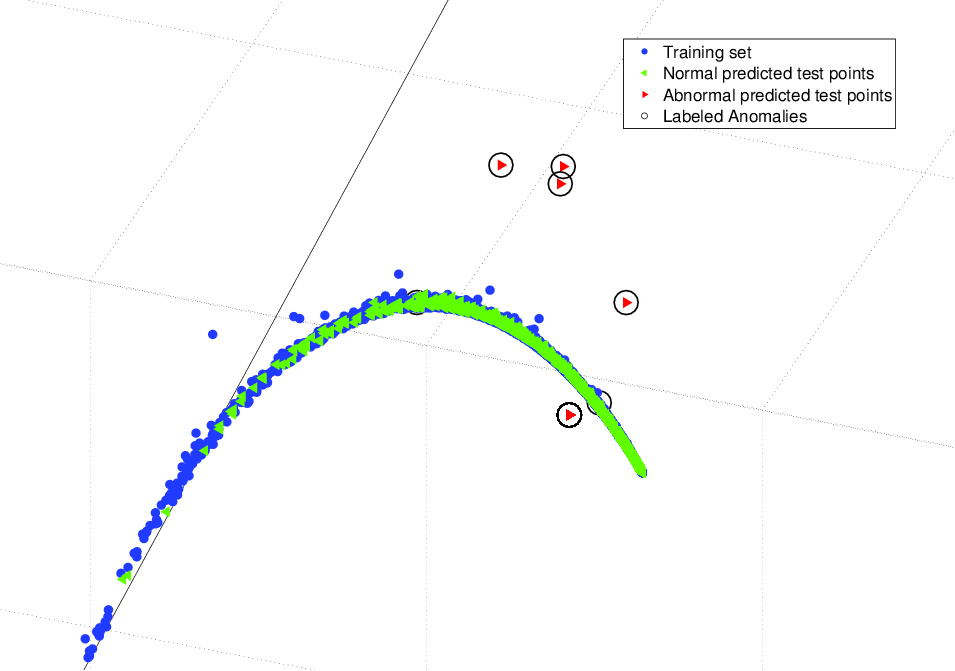}
\label{fig:darpa_normal}}
\subfigure[]{
\includegraphics[scale=.15]{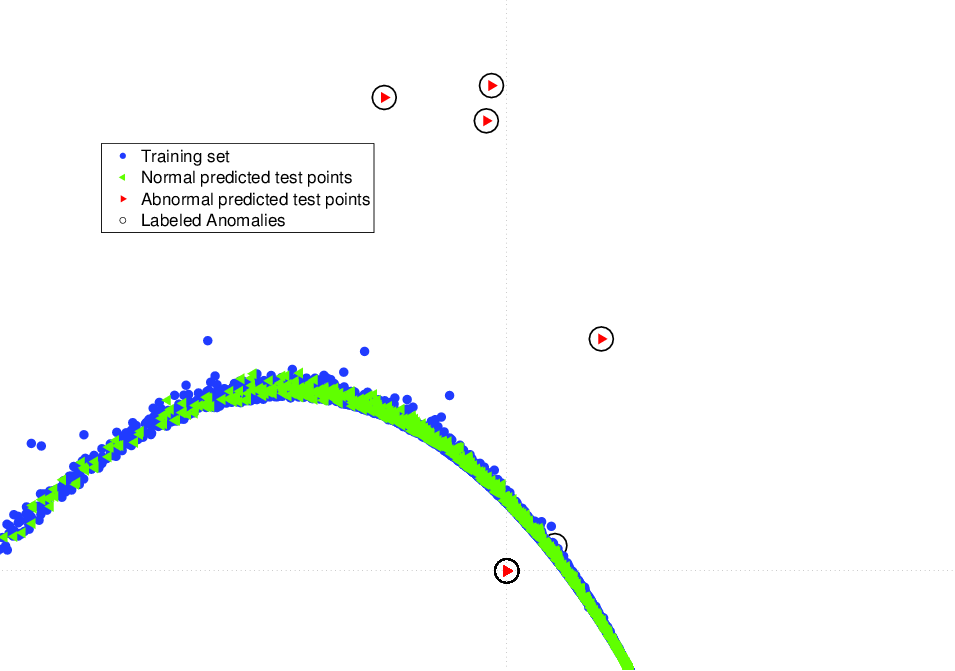}
\label{fig:darpa_anomal}} \caption{Three different angles and scales of an aligned version of a three-dimensional projection of $\timel{\map h}{1}_s$, applied to the training set $\set X$ (blue points), its out-of-sample extension to $\set T_{thu}\backslash\set F$ (strictly-normal data points are green, strictly-abnormal data points are red), and labeled anomalies (black circles).~\subref{fig:darpa_general}~general view.~\subref{fig:darpa_normal} strictly-normal out-of-sample data points are mapped closely to the training set.~\subref{fig:darpa_anomal}~strictly-abnormal data points.}
\label{fig:darpa_anomalies}
\end{figure}

As in Section~\ref{sssec: lde}, the alignment was done only for visualization purposes using the algorithm in Appendix~\ref{appx}.

\subsection{Semi-supervised multi-class classification of high dimensional data}
\label{subsec:isolet}

In this section, a multi classification process, based on Algorithms~\ref{alg: q-less} and~\ref{alg: q-less_oos}, is presented. The analyzed data is the parametric $m$-dimensional ISOLET dataset~\cite{Bache+Lichman:2013} $\bar {\set X}\subset\Rn{m}$ where $m=617$ that contains $7,797$ data points in $[-1,1]^m$. Each data point corresponds to a single human pronunciation of ISOlated LETters. The goal is to classify a testing subset ${\set T}\subset \bar {\set X}$, which contains $1,559$ letter-pronunciation samples spoken by $30$ people, to $26$ classes that are  based on a training set $\set X\subset\bar{\set X}$ that contains $n=6,238$ samples that were already classified to $\set X_\zeta$, $\zeta\in\set I\triangleq\{A,B,\ldots,Z\}$, spoken by $120$ different people.
For this purpose, Algorithm~\ref{alg: q-less} was applied to each training set $\set X_\zeta$ with a distortion parameter $\mu=4.7$, to produce $26$ dictionaries $\set D_\zeta\subset\set X_\zeta$, $\zeta\in\set I$.  Then, for each of these dictionaries, Algorithm~\ref{alg: q-less_oos} was applied to each element in the testing set $\vect x\in\set T$, to produce its distortion rate $\bmu_\zeta(\vect x)$ (see Eq.~\ref{eq:dist_func}). Finally, each of the testing data points was classified to a class whose dictionary described it best, i.e. $\vect x$ was classified to class $\zeta_0$, where $\zeta_0\triangleq\arg\min_{\zeta\in\set I}\bmu_\zeta(\vect x)$. The parameter $\mu$ was determined by taking part of the training set to serve as a validation set. Then, several values of $\mu$ were applied on the (reduced) training set, and the validation set was classified based upon these values. The chosen $\mu$ was the one that was optimal on the validation set.
The results are shown in Fig.~\ref{fig:isolet_icpqr}. Out of $1,559$ test samples, $92\%$ were classified correctly. The classification of the test data is presented in a confusion matrix in Fig.~\ref{fig:isolet_icpqr}. By looking at the shades of the diagonal it can be seen that the majority of the test samples of each class were classified correctly. The sets $\{B,C,D,E,G,P,T,V,Z\}$ and $\{M,N\}$ are the most difficult letters to classify due to high similarity in pronunciation of these letters within each set. The state-of-the-art classification accuracy of this dataset is  $96.73\%$. It  was achieved in~\cite{Dietterich91error-correctingoutput} by using  $30$-bit error correcting output codes that is  based on neural networks. This method is far more complex than the solution proposed in this work.
\begin{figure}[H]
\centering
\includegraphics[scale=.31]{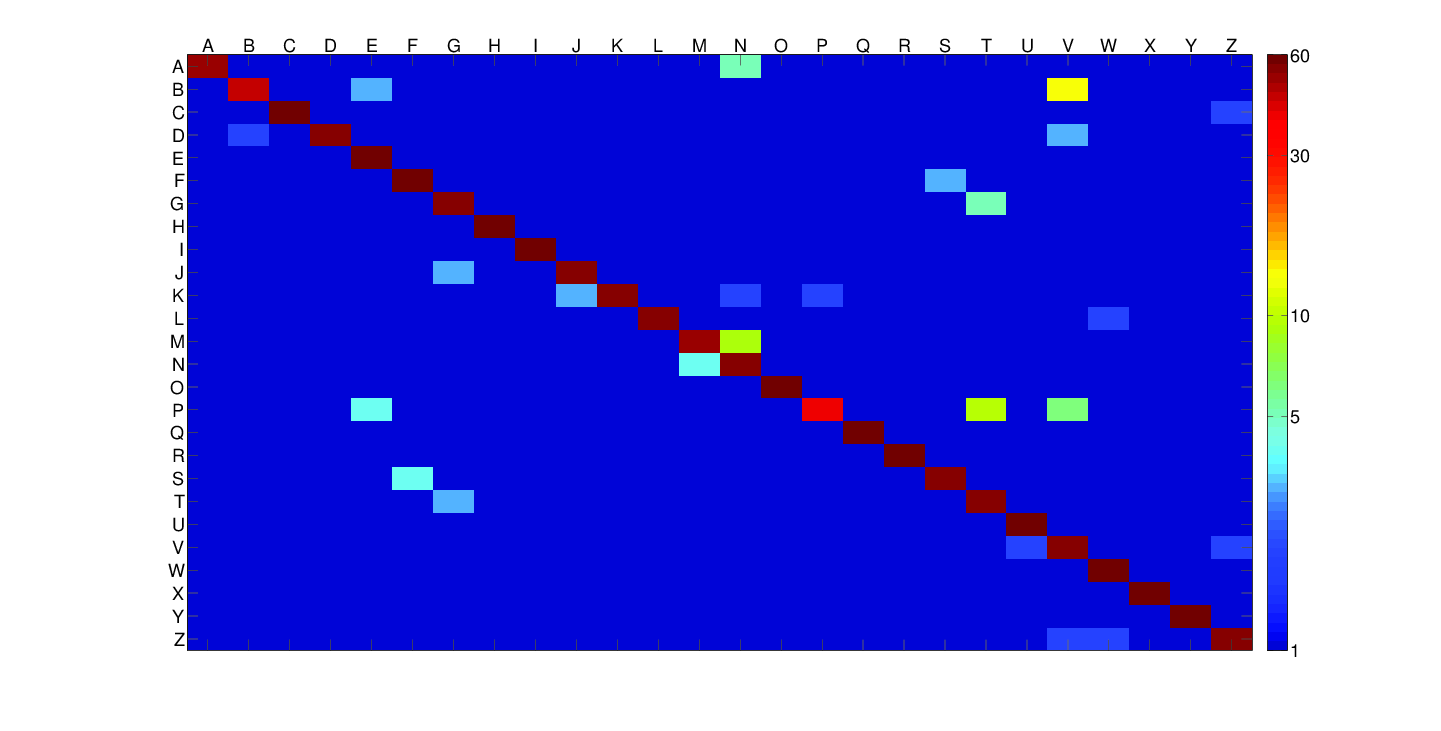}
\caption{Classification of the test samples in the ISOLET dataset. For each ordered couple $(i,j)$, the cell at row $i$ and column $j$ is colored according to the number of test samples belonging to class $i$ that were classified by the algorithm to class $j$. The cells on the diagonal denote correct classification. For each letter, a total of $60$ test samples are provided (except for 'M' which is missing one sample due to recording difficulties).}
\label{fig:isolet_icpqr}
\end{figure}

\section{Conclusions and Future Works}~\label{sec:conclusions}
This work presents a complete framework for linear dimensionality reduction, out-of-sample extension and anomaly detection algorithms for high-dimensional parametric data, which is based on incomplete pivoted QR decomposition of the associated data matrix. The presented method preserves the high-dimensional geometry of the data up to a user-defined distortion parameter. Such a low-dimensional data representation enables further geometrically-based data analysis which, due to the geometry-preservation, is still valid to the original data. The storage complexity of the method is extremely low compared to the classical PCA. In the worst case, its computational complexity is similar to that of the PCA. The method provides a dictionary, which is a subset of landmark data points, that forms a basis for the projection of the entire data to a low dimensional space. Out-of-sample extension and anomalous data point detection become simple tasks once the dictionary is computed. Although the suggested method is designated for parametric data analysis, in some cases it can be adapted to non-parametric data analysis frameworks, as was demonstrated for the DM framework. The stability of our method to perturbations (noise) was proved and its connection to matrix-approximation was presented.

Experimental results show that our method achieves a  lower dimensional embedding than PCA achieves for a certain distortion. Moreover, analysis of both synthetic and real-world datasets achieved good performance for  dimensionality reduction, out-of-sample extension, anomaly detection and multi-class classification.

Future work includes randomized version of the presented framework to provide a more computationally efficient method for a dictionary-based dimensionality reduction method. In addition, a generalization of our method for non-parametric data analysis methods should be considered, as well as a dynamic framework that enables to cope with datasets that vary across time. Finally, a parallel version of the algorithm, which simultaneously builds dictionaries for subsets of the data and then unifies them, is planned.

\appendix
\section{Appendix: Alignment Algorithm}\label{appx}
This section details the alignment algorithm that was utilized in Sections~\ref{sssec: lde} and~\ref{subsec:exp_darpa}, for visualization purposes.

Suppose that $A$ and $B$ are two $m\times n$ matrices of $n$ data points in $\Rn{m}$. If the sizes of $A$ and $B$ are different, then padding by zeros is performed. Clearly, it will not change the geometry of the columns of these matrices. Let $\bar A$ and $\bar B$ be the centralized versions of $A$ and $B$ around the columns means of each one of them. Then, the best orthogonal alignment of $\bar B$'s columns with $\bar A$'s columns is provided by the orthogonal matrix $Q=U_AU_B^\ast$, where $U_A$ and $U_B$ are the left singular vectors of $A$ and $B$, respectively. Then, the aligned centralized matrix $\tilde B = Q*B$ is decentralized by $A$'s columns mean. Obviously, since $Q$ is orthogonal, the geometry of  $\tilde B$'s columns is unaffected. Algorithm~\ref{alg: align} describes the above.

\IncMargin{1em}
\begin{algorithm}
    \DontPrintSemicolon
    \SetAlgoLined
    \SetKwComment{tcp}{//}{}
    \SetKwInOut{Input}{Input}\SetKwInOut{Output}{Output}
    \Input{Two $m\times n$ matrices $A$ and $B$}
    \Output{An $m\times n$ matrix $\tilde B$ whose columns geometry is identical to that of $B$'s columns, and  $\tilde B$'s columns are optimally aligned with  $A$'s columns. }
    \BlankLine
    centralize $A$: $\bar A = A-\vect a\vect 1_n^\ast$, where $\vect a = \sum_{j=1}^n\cols{A}{j}$ and $\vect 1_n\in\Rn{n}$ is the all-ones vector\\
    centralize $B$: $\bar B = B-\vect b\vect 1_n^\ast$, where $\vect b = \sum_{j=1}^n\cols{B}{j}$\\
    compute the left singular vectors $U_A$ and $U_B$ of $\bar A$ and $\bar B$, respectively\\
    define $\tilde B = U_AU_B^\ast\bar B+\vect a\vect 1_n^\ast$

    \caption{Alignment Algorithm}
    \label{alg: align}
\end{algorithm}
\DecMargin{1em}

\section*{Acknowledgments}
\noindent This research was partially supported by the Israeli Ministry of Science \& Technology (Grants No. 3-9096, 3-10898), US-Israel Binational Science Foundation (BSF 2012282) and Blavatnik Computer Science Research Fund.

\bibliographystyle{plain}
\bibliography{qrDR}
\end{document}